\documentclass{article}

\usepackage[nonatbib,preprint]{neurips_2026}

\usepackage[utf8]{inputenc} 
\usepackage{multirow}% encodage des caracteres utilise (pour les caracteres accentues) -- non utilise ici.
\usepackage{amsmath,amssymb,amsthm}	
\usepackage{thm-restate}% pour les maths
\usepackage{dsfont}  % Better fonts
\usepackage{mathrsfs}   
\usepackage{mathtools}
\usepackage{bbm}
\usepackage{ifthen}
\usepackage{graphicx}
\usepackage{subcaption}
\usepackage{booktabs}
\graphicspath{ {./images/} }  % include graphs
\usepackage{hyperref}	%references links 
\usepackage{algorithm}
\usepackage{algpseudocode}
\usepackage{color}
\usepackage[dvipsnames]{xcolor}

\usepackage{enumitem}

\usepackage{hyperref}	%references links 

\usepackage{color}
\usepackage[dvipsnames]{xcolor}

\usepackage{lscape}
\usepackage{parskip}

\usepackage{todonotes}

\newtheorem{lemma}{Lemma}
\newtheorem{theorem}{Theorem}
\newtheorem{proposition}{Proposition}
\newtheorem{definition}{Definition}
\newtheorem{remark}{Remark}

\theoremstyle{remark}

\theoremstyle{plain}
\newtheorem{assumption}{Assumption}

\theoremstyle{plain}
\newtheorem{corollary}{Corollary}

\newcommand{\Expectation}[1][]{ 
    \ifthenelse{ \equal{#1}{} }
    {\mathbb{E}}
    {\mathbb{E} \left [ #1 \right ] }
}
\newcommand{\Proba}[1][]{ 
    \ifthenelse{ \equal{#1}{} }
    {\mathbb{P} }
    {\mathbb{P} \left ( #1 \right )}
}
\newcommand{\Probatilde}[1][]{ 
    \ifthenelse{ \equal{#1}{} }
    {\tilde{\mathbb{P}} }
    {\tilde{\mathbb{P}} \left ( #1 \right )}
}

\newcommand{\indicator}[1][]{ 
    \ifthenelse{ \equal{#1}{} }
    {\mathds{1} }
    {\mathds{1} \left \{ #1 \right \}}
}

\newcommand\numberthis{\addtocounter{equation}{1}\tag{\theequation}}

\usepackage{eurosym}
  
\usepackage[
backend=biber,
style=apa,natbib=true
]{biblatex}
\usepackage[utf8]{inputenc} % allow utf-8 input
\usepackage[T1]{fontenc}    % use 8-bit T1 fonts
\usepackage{hyperref}       % hyperlinks
\usepackage{url}            % simple URL typesetting
\usepackage{booktabs}       % professional-quality tables
\usepackage{nicefrac}       % compact symbols for 1/2, etc.
\usepackage{microtype}      % microtypography
\usepackage{xcolor}

\title{Instance-dependent Stochastic Lipschitz bandit}

\author{Marius Potfer$^{1,2}$ \\
\And Vianney Perchet$^{1,3}$}

\begin{document}

\maketitle

\begin{center}
\begin{tabular}{l}
   $^1$ Crest (Fairplay joint team), ENSAE\\
   $^2$ EDF R\&D \\
   $^3$ Criteo AI Lab
\end{tabular}

\end{center}

\begin{abstract}
    We study the Lipschitz bandit problem, where a learner sequentially maximizes an unknown Lipschitz function $f$ over a domain $\mathcal{X} \subset [0,1]^d$ using noisy pointwise evaluations. Existing regret bounds are either worst-case, scaling as $\tilde{\Theta} \left ( T^{\nicefrac{d+1}{d+2}}\right )$, or adaptive via the zooming dimension $d_z$, yielding $\tilde{\Theta} \left ( T^{\nicefrac{d_z+1}{d_z+2}}\right )$. However, such zooming-based guarantees are only partially instance-dependent, as they depend solely on the asymptotic growth of near-optimal level sets and fail to capture finer structural properties of $f$. We provide an analysis and an algorithm that characterizes the regret through integrals of the suboptimality gap of $f$ over its level sets. This yields regret bounds that adapt to the local growth of level sets, rather than only their asymptotic behavior. As a corollary, when the set of maximizers has dimension $d^\star>0$, we obtain improved adaptive rates of order $\tilde{\mathcal{O}} \left ( T^{\nicefrac{d_z+1}{\max(d_z,d^\star)+2}}\right )$ strictly improving over classical zooming bounds in this regime. Finally, we extend our analysis to the full-information setting (Lipschitz experts) and show how some of the regularity assumptions can be relaxed.
\end{abstract}

%\input{Prospective_struture}

% (no filepath)
\section{Introduction}\label{sec : introduction}

Lipschitz bandit algorithms are a natural tool for sequential optimization over continuous action spaces. They have found applications in black-box hyperparameter optimization and related parameter search problems, as explored for instance by \citet{li2023pyxab} and \citet{feng2023lipschitz}, and more broadly in tree-based optimization approaches such as \citet{shi2024optimizing} and \citet{qiu2024tree}. They also arise in online learning settings such as repeated auctions and dynamic pricing, as studied by \citet{weed2016online}, \citet{branzei2023learning}, and \citet{baltaoglu2017online}, and in pricing models such as \citet{bu2022context}. In these applications, the effective difficulty of the instance can vary strongly with the geometry of the reward function and of the action space available to learners. This is precisely where instance-dependent bounds become useful: they capture how local structure influences the performance of algorithms, while worst-case analyses typically do not.

We therefore consider the stochastic Lipschitz bandit setting, which includes most of the above-described cases, and focus on obtaining instance-dependent regret bounds. The action set is bounded, $\mathcal{X}\subset[0,1]^d$, and the learner interacts over $T\in\mathbb{N}$ rounds. At each round $t\in[T]$, it selects $x_t\in\mathcal{X}$ and observes $Y_t=f(x_t)+\eta_t$, where $f:\mathcal{X}\to\mathbb{R}$ is an unknown mean-reward function and $(\eta_t)_t$ is zero-mean, conditionally $1$-sub-Gaussian. We assume $f$ is $L$-Lipschitz, i.e.\ $|f(x)-f(y)|\le L\|x-y\|_\infty$ for all $x,y\in\mathcal{X}$, and let $f^\star=\max_{x\in\mathcal{X}} f(x)$. The goal of the learner is to minimize the expected regret
\[
R_T=\mathbb{E}\!\left[\sum_{t=1}^T \bigl(f^\star-f(x_t)\bigr)\right].
\]

Known analyses quantify performance by regret growth rates. In the worst case, $R_T$ scales as $\tilde\Theta\!\bigl(T^{\frac{d+1}{d+2}}\bigr)$ for Lipschitz/continuum-armed bandits \citep{kleinberg2004nearly,auer2007improved}. Adaptive bounds replace $d$ by the zooming or near-optimality dimension \citep{kleinberg2008multi,bubeck2011x}; see also the survey \citep{slivkins2019introduction}. However, these complexity parameters are asymptotic and do not provide an explicit finite-horizon characterization of how $R_T$ depends on a specific instance $(\mathcal{X},f)$. This paper fills that gap by giving a non-asymptotic, instance-dependent analysis that makes the dependence on the geometry of $\mathcal{X}$ and the structure of $f$ explicit.

\subsection{Contributions}
We analyze stochastic Lipschitz bandits and Lipschitz experts and derive explicit, non-asymptotic instance-dependent regret bounds. Our main results are stated in an integral (truncated) form that depends directly on the geometry of the level sets of $\Delta(x)=f^\star-f(x)$, and therefore on the specific instance $(\mathcal{X},f)$ rather than only on an asymptotic exponent.

We complement the upper bounds with almost matching lower bounds for the outer integral term, up to logarithmic factors, showing that the instance-dependent characterization is essentially tight. As a corollary, we obtain refined zooming-type rates that incorporate the dimension $d^\star$ of the maximizer set, yielding $\tilde O\!\left(T^{\frac{d_z+1}{\max(d_z,d^\star)+2}}\right)$ and strictly improving classical bounds when $d^\star>d_z$. We also provide a theoretical relation between $d_z$ and $d^\star$ under mild geometric assumptions to quantify the scale of improvements possible when $d^\star>d_z$. A similar corollary is derived for the full-information expert setting, demonstrating that the zooming dimension also dictates the regret improvements when $d^\star > d_z$ under full information.

Finally, we present implementable hierarchical optimistic algorithms (PACO for bandits and SOUS for experts) that achieve the stated guarantees, and we extend the analysis to relaxed regularity assumptions, including one-sided Lipschitzness, with corresponding regret bounds.

\subsection{Related work}
In stochastic $K$-armed bandits, one typically distinguishes worst-case guarantees from instance-dependent bounds that adapt to suboptimality gaps; see, e.g., \citet{lai1985asymptotically, lattimore2020bandit}. Our goal is an analogous instance-dependent characterization for Lipschitz bandits (although in Lipschitz bandits, "instance-dependent" is more nuanced, as discussed in the review of \cite{slivkins2019introduction}), where the finite list of gaps is replaced by the geometry of near-optimal regions in a metric space.

Continuum-armed/Lipschitz bandits are well studied with worst-case regret rates $\tilde{\Theta} \left ( T^{\nicefrac{d+1}{d+2}}\right )$ established in \citet{kleinberg2004nearly,auer2007improved,bubeck2011x} and refined via the zooming algorithm \citep{kleinberg2008multi}. Adaptive guarantees of $\tilde{\Theta} \left ( T^{\nicefrac{d_z+1}{d_z+2}}\right )$ based on the zooming or near-optimality dimension appear in \citet{kleinberg2008multi,bubeck2011x,kleinberg2019bandits}. As emphasized in the survey of \citet{slivkins2019introduction}, these complexity parameters are instance-dependent; yet as we highlight they depend on asymptotic quantities and do not yield the desired precise and explicit finite-horizon dependence on a specific instance $(\mathcal{X},f)$, which motivates our instance-dependent integral characterization.

Several extensions have been studied beyond the purely stochastic setting, including discrete Lipschitz bandits \citep{magureanu2014lipschitz}, adversarial or corrupted models \citep{podimata2021adaptive,kang2023robust}, batched feedback \citep{feng2022lipschitz}, federated variants \citep{li2024federated}, and nonstationary environments \citep{nguyen2025nonstationary}. Kernelized bandits are a related continuous-armed model \citep{chowdhury2017kernelized,chatterji2019online,hong2023optimization}, and instance-dependent bounds have been obtained there in \citet{shekhar2022instance}. Full-information metric-space experts are treated in \citet{kleinberg2019bandits}; our work provides explicit instance-dependent guarantees for this Lipschitz experts setting as well. Finally, Lipschitz bandit ideas have been used in black-box optimization and hyperparameter search \citep{li2023pyxab,feng2023lipschitz,shi2024optimizing,qiu2024tree}, and recent work studies computational refinements \citep{zhu2025lipschitz}.

\subsection{Setting and preliminaries}\label{sec:setting}
We study the problem of maximizing an unknown Lipschitz function from noisy zero-order feedback. The learner interacts with a stochastic environment over a horizon $T$, on a bounded action space $\mathcal{X}\subset[0,1]^d$ with a non-empty interior. We write $[T]=\{1,\ldots,T\}$ and use the metric $d_\infty(x,y)=\|x-y\|_\infty$. At each round $t\in[T]$, the learner selects an action $x_t\in\mathcal{X}$ and observes
\[
Y_t = f(x_t) + \eta_t,
\]
where $f:\mathcal{X}\to\mathbb{R}$ is an unknown mean-reward function and $(\eta_t)_{t\ge1}$ are independent, zero-mean, $1$-sub-Gaussian noise variables. We assume that $f$ attains its maximum on $\mathcal{X}$ and write $f^\star=\max_{x\in\mathcal{X}} f(x)$ and $\mathcal{X}^\star=\arg\max_{x\in\mathcal{X}} f(x)$ for the set of maximizers. The performance of an algorithm is measured by the expected cumulative regret
\[
R_T := \mathbb{E}\!\left[\sum_{t=1}^T \bigl(f^\star-f(x_t)\bigr)\right].
\]
For convenience we also introduce the gap function $\Delta(x):=f^\star-f(x)$.

\begin{assumption}[Lipschitz continuity]\label{assump:Lipschitz}
There exists $l>0$ such that for all $x,y\in\mathcal{X}$, $|f(x)-f(y)|\le l\|x-y\|_\infty$. The learner is given an upper bound $L\ge l$.
\end{assumption}

We assume the upper bound $L$ is known to the learner; this is standard, and we refer to \citet{bubeck2011lipschitz,valko2013stochastic} for techniques that lift this requirement. Motivated by online learning in auction and dynamic pricing, we also discuss extensions to weaker regularity conditions which are more suited to these settings in \autoref{sec : relaxed}.

We refer to the pair $(\mathcal{X},f)$ as a \emph{problem instance} and will provide a precise characterization of how the regret scales with the instance. To that end, for any $r>0$ we define the near-optimal sets
\[
\mathcal{X}_r := \{x\in\mathcal{X}: f^\star - f(x)\le r\},
\]
and the corresponding level-set annuli $\mathcal{X}_r\setminus \mathcal{X}_{r/2}$.

\paragraph{Packing and covering numbers.}
To handle the continuous nature of $\mathcal{X}$, our algorithms discretize subsets of $\mathcal{X}$. We account for the effect of the discretization on regret, by using two standard complexity measures: packing and covering numbers. For $x \in \mathcal{X}$ and $\epsilon>0$, let $B_\infty(x,\epsilon)=\{y\in\mathbb{R}^d:\|y-x\|_\infty\le\epsilon\}$. The packing and covering numbers of $\mathcal{Y}\subseteq\mathcal{X}$ at scale $\epsilon$ are denoted $\mathcal{N}(\mathcal{Y},\epsilon)$ and $\mathcal{M}(\mathcal{Y},\epsilon)$ and defined by
\begin{align*}
\mathcal{N}(\mathcal{Y},\epsilon)
&:= \sup\Bigl\{|S|:\ S\subseteq\mathcal{Y},\ \min_{x\neq x'\in S}\|x-x'\|_\infty > \epsilon\Bigr\},\\
\mathcal{M}(\mathcal{Y},\epsilon)
&:= \inf\Bigl\{m:\ \exists x_1,\dots,x_m,\ \mathcal{Y}\subseteq \cup_{i=1}^m B_\infty(x_i,\epsilon)\Bigr\}.
\end{align*}

These are standard notions; see, for instance, Chapter 4 of \citet{vershynin2018high}. They satisfy $\mathcal{N}(\mathcal{Y},2\epsilon)\le \mathcal{M}(\mathcal{Y},\epsilon)\le \mathcal{N}(\mathcal{Y},\epsilon)$, and since $\mathcal{X}\subset[0,1]^d$ we have $\mathcal{M}(\mathcal{Y},\epsilon)\le (1/\epsilon)^d$.

\paragraph{Zooming dimension.}
\citet{kleinberg2008multi} defines the zooming dimension to quantify the metric complexity of near-optimal actions.
The \emph{$C$-zooming dimension} of an instance $(\mathcal{X},f)$ is the smallest $d_z\ge 0$ such that, for every $r\in(0,1]$,
\[
\mathcal{M}\!\left(\mathcal{X}_r\setminus\mathcal{X}_{r/2},\, r/16\right)\ \le\ C\,r^{-d_z}.
\]

\section{Algorithms and Instance-Dependent Bounds}\label{sec:algo and inst bounds}

To derive our instance-dependent upper bounds, we first present the algorithm for which we will prove these guarantees. The algorithm is inspired by HOO \citep{bubeck2011x} and operates based on a phased procedure: at phase $k$ it constructs a discretization of the current ``active'' region at resolution $r_k := 2^{-k}$, allocates samples to this discretization via a successive-elimination subroutine, and then shrinks the active region to the union of balls around the surviving actions. We refer to the resulting method as \emph{Phased Adaptive Covering Optimization (PACO)}.

\subsection{Discretization Oracle}
We assume to have access to an oracle that provides said discretization; given a set $\mathcal{Y}\subseteq[0,1]^d$ and a radius $r\in(0,1]$, it returns a finite set of points $\mathbf{O}(\mathcal{Y},r)=\{y_1,\ldots,y_n\}\subseteq \mathcal{Y}$ such that $\mathcal{Y}$ is covered by $\ell_\infty$-balls of radius $r$ centered at these points:
\[
\mathcal{Y}\subseteq \bigcup_{i=1}^n B_\infty(y_i,r).
\]

We assume the following quality condition on our oracle:

\begin{assumption} \label{assump:good oracle}
There exist constants $c_{\mathrm{sep}}\in(0,1]$ and $C_{\mathrm{net}}\ge 1$ such that for any $\mathcal{Y}\subset[0,1]^d$ and any $r\in(0,1]$, the set $\mathbf{O}(\mathcal{Y},r)$ satisfies:
\begin{itemize}
    \item \textbf{Separation} for any distinct $y,y'\in \mathbf{O}(\mathcal{Y},r)$, $\|y-y'\|_\infty \ge c_{\mathrm{sep}}\, r$.
    \item \textbf{Cardinality:} $|\mathbf{O}(\mathcal{Y},r)| \le C_{\mathrm{net}}\, \mathcal{M}(\mathcal{Y},r)$.
\end{itemize}
\end{assumption}

\autoref{assump:good oracle} is mild; for example, it is satisfied by the net obtained via the greedy procedure that repeatedly adds a point at distance $>r$ from the current set until no such point remains. This construction and its properties have been well studied and are, for instance, discussed in \cite{vershynin2018high}, Chapter 4.

\subsection{Algorithm}

Before fully stating the algorithm, we define a few required quantities. Since our algorithm proceeds in phases, let $\tau_k$ denote the duration (number of time steps) of phase $k$, and let $s_k := \sum_{j=1}^k \tau_j$ denote the absolute end time of phase $k$ (with $s_0:=0$). For any round $\ell\ge 1$ within phase $k$, let $s_{k,\ell}$ denote the absolute time step at which the $\ell$-th round of phase $k$ ends. The empirical mean of an action $a\in\mathcal{X}$ after $\ell$ pulls within phase $k$ is
\[
\widehat{\mu}_{k,\ell}(a)
\;:=\;
\frac{1}{\ell}\sum_{t=s_{k-1}+1}^{s_{k,\ell}} Y_t\,\mathbf{1}\{x_t=a\}.
\]

The main intuition behind PACO is that, when concentration holds, the active region at phase $k$ is essentially the $r_k$-near-optimal region, so that the discretization $\mathbf{O}(\mathcal{X}_{r_k},r_k)$ adapts to the metric complexity of near-optimal sets.

\begin{algorithm}[htbp]
\caption{PACO (high-level, with confidence budget)}\label{algorithm: high level}
\begin{algorithmic}[1]
\State \textbf{Input:} time horizon $T$, action set $\mathcal{X}$, confidence level $\delta\in(0,1)$, Lipschitz constant $L$
\State \textbf{Initialize:} $t\gets 1$, $k\gets 1$, active region $\mathcal{A}_1 \gets \mathcal{X}$
\While{$t \le T$}
    \State $r_k \gets 2^{-k}$
    \State $\delta_k \gets \frac{6\delta}{\pi^2 k^2}$
    \State Discretize: $S_k \gets \mathbf{O}(\mathcal{A}_k, \nicefrac{r_k}{L})$
    \State Run Algorithm~\ref{algo : successive elimination} on $S_k$ with parameters $(r_k,\delta_k, t, T)$, obtain survivors $\widehat{S}_k$ and updated $t$
    \State Update active region: $\mathcal{A}_{k+1} \gets \mathcal{X} \cap \bigcup_{a\in\widehat{S}_k} B_\infty(a,\nicefrac{r_k}{L}) $
    \State \Comment{$t$ is updated inside Algorithm~\ref{algo : successive elimination}}
    \State $k\gets k+1$
\EndWhile
\end{algorithmic}
\end{algorithm}

\begin{algorithm}[htbp]
\caption{Successive elimination up to accuracy $r_k=2^{-k}$ (anytime bounds)}\label{algo : successive elimination}
\begin{algorithmic}[1]
\State \textbf{Input:} finite set of arms $S_k$, accuracy $r_k\in(0,1]$, confidence budget $\delta_k\in(0,1)$, time $t$, horizon $T$
\State \textbf{Output:} surviving set $\widehat{S}_k$, updated $t$
\State \textbf{Initialize:} $\ell\gets 1$, $A_\ell \gets S_k$
\While{$u_{k,\ell} > r_k/4$ \textbf{and} $t \le T$}
    \State Pull each $a\in A_\ell$ once (abort if $t > T$); update empirical means $\widehat{\mu}_\ell(a)$ and $t \gets t + 1$
    \State \Comment{For any $a\in A_\ell$, it has been pulled exactly $\ell$ times so far.}
    \State Set confidence radius:
    \State $u_{k,\ell} \gets \sqrt{\frac{2}{\ell}\log\!\Big(\frac{\pi^2\,\ell^2\,|S_k|}{6\,\delta_k}\Big)} $
    \State Update the active set:
    \State $A_{\ell+1} \gets \left\{a\in A_\ell:\ \widehat{\mu}_\ell(a)+u_{k,\ell} \ge \max_{j\in A_\ell}\widehat{\mu}_\ell(j)-u_{k,\ell} -r_k\right\}$
    \State $\ell\gets \ell+1$
\EndWhile
\State \Return $\widehat{S}_k \gets A_\ell$
\end{algorithmic}
\end{algorithm}

\begin{restatable}{lemma}{LemmaInstanceRegret}\label{lemma:first regret bounds}
Under Assumptions~\ref{assump:Lipschitz} and~\ref{assump:good oracle}, when running Algorithm~\ref{algorithm: high level} in conjunction with Algorithm~\ref{algo : successive elimination} on a Lipschitz bandit instance $(\mathcal{X},f)$, there exist universal constants $c_1$ and $c_2$ such that the regret satisfies
\[
R_T \leq  c_1\sum_{k=1}^{k_T} 2^k \mathcal{N} \left(\mathcal{X}_{2^{-k}}\setminus \mathcal{X}_{2^{-(k+1)}},2^{-k}/L\right) + c_2 2^{-(k_T+1)}\sum_{k=1}^{k_T} 2^{2k} \mathcal{N} \left ( \mathcal{X}_{2^{-(k_T+1)}} \setminus \mathcal{X}^\star, 2^{-k}/L \right )
\]
where $k_T$ is the phase index containing the horizon $T$ (i.e., $s_{k_T-1} < T \le s_{k_T}$).
Moreover, $k_T$ is characterized by the sampling budget of the preceding completed phase, yielding the following bound
\begin{equation*}
    T \geq c\, 2^{2(k_T-2)} \mathcal{M} \left ( \mathcal{X}_{r_{k_T}},2^{-(k_T-2)}/L \right )
\end{equation*}
for a constant $c>0${} (when $k_T \ge 2$).
\end{restatable}

\begin{proof}[Proof  insight]
Our regret bound, at its core, relies on a set of concentration inequalities, and analyzing the regret in a global high probability \textit{good event} for which the empirical means are close to the true means. Combined with \autoref{assump:Lipschitz}, we show in \autoref{lem:nearopt-never-elim} that the maximizing set $\mathcal{X}^\star$ always remains in the active sets $\mathcal{A}^t$. We then leverage our hypothesis on the oracle \autoref{assump:good oracle} as well as standard techniques to bound how many times a point $a \in S_k$ can be queried during phase $k$, depending on its sub-optimality $\Delta(a)$ to obtain the regret bounds above. The full proof of \autoref{lemma:first regret bounds} and the detailed steps to obtain the bounds are in \autoref{app:proof-first-lemma-complete}.
\end{proof}

\noindent

Lemma~\ref{lemma:first regret bounds} is a direct expression of the instance-dependent control we obtain. Yet, this bound is hardly practical; because it directly reflects our use of adaptive discretization via our oracle, it states the regret as a sum of packing numbers over level-sets of $f$. This makes it difficult to extract sharp insights about how both the continuous nature of the problem and the geometry of $(\mathcal{X},f)$ influence the behavior of the algorithm.

\subsection{Integral bounds}

To better showcase the behaviour of the regret under \autoref{algorithm: high level}, we state below \autoref{theorem:instance dependant integral regret bounds}, which presents an upper bound on the regret as an integral of $\Delta$ over level-sets of $f$.
Both this bound and the techniques we used to derive it are inspired by techniques from online Lipschitz optimization (noise free) presented in \cite{bachoc2021instance,de2024certified}.

The following assumption is necessary for the packing numbers to be well behaved and allows for a clean link with integral bounds:
\begin{assumption}\label{assump: nice opti sets main}
There exists $l \in \mathbb{N}$ and $\gamma \in (0,1]$ such that for all $k\geq l$ and all $x\in \mathcal{X}_{2^{-k}}$,
\[
\mathrm{vol}\left(B_\infty\left(x,\nicefrac{2^{-(k)}}{L}\right)\cap \mathcal{X}\right)\ \ge\ \gamma\, v_{2^{-(k)}/L}.
\]
Where $v_{r}$ denotes the volume of a ball of radius $r$ in $\mathbb{R}^d$.
\end{assumption}
In the following, we assume that \autoref{assump: nice opti sets main} is true as of rank $l=0$.(Assuming $l>0$ only adds a constant factor in front of the integral in the following result). Note that this is only a mild geometric assumption, at a high level, we only require that a small fraction of each ball centered around any close-to-optimal point be included in $\mathcal{X}$. Such assumptions are common in Lipschitz function analysis \citep{bachoc2021instance,hu2020smooth}.

\begin{restatable}{theorem}{theoremRegretInstanceIntegral}
\label{theorem:instance dependant integral regret bounds}
When running Algorithm~\ref{algorithm: high level} in combination with Algorithm~\ref{algo : successive elimination} with confidence parameter $\delta=T^{-3}$ on a Lipschitz bandit instance $(\mathcal{X},f)$, the regret is bounded as follows:
\[
R_T \;\le\;
C_R\log(T)\int_{\mathcal{X}\setminus \mathcal{X}^\star}
\frac{dx}{\max\left(\Delta(x),2^{-k_T}\right)^{d+1}},
\]

where the following bound on $k_T$ applies:  $ c\, 2^{2(k_T-2)} \mathcal{M} \left ( \mathcal{X}_{r_{k_T}},2^{-(k_T-2)}/L \right ) \leq T$.

\end{restatable}

\begin{proof}[Proof  insight]
The key element in this proof comes from the link between packing numbers and our integral of the gap-function. As in \cite{bachoc2021instance,de2024certified}, by using \autoref{assump: nice opti sets main}, we can provide such a bound by using a peeling argument over the level sets of $f$ and upper bounding volume integrals with packing numbers at scale $r$ times the volume of the corresponding $\ell_\infty$ balls of radius $r$. This bound is formally provided in \autoref{lemma : sum integral eq} in the appendix.
The full proof is provided in \autoref{app: proof theorem 1}.
\end{proof}

\begin{remark}\label{remark : regret two part}
The upper bound in \autoref{theorem:instance dependant integral regret bounds} is a truncated integral. When $\Delta(x)\ge 2^{-k_T}$, the integrand is $\Delta(x)^{-(d+1)}$ and corresponds to the cost of ruling out clearly sub-optimal regions. When $\Delta(x)<2^{-k_T}$, the integrand is capped at $2^{(d+1)k_T}$ and accounts for the near-optimal region that cannot be resolved by time $T$. 
\end{remark}

\subsection{Matching lower bounds}
We show below that the above instance-dependent upper bounds are accompanied by lower bounds that match the outer integral term up to logarithmic factors. We present the resulting lower bound directly in terms of integrals over $\mathcal{X}$ in order to facilitate comparison with Theorem~\ref{theorem:instance dependant integral regret bounds}.
The lower bound means the following: for every algorithm in the class, there are instances on which the expected regret cannot be smaller than a quantity that depends on $f$ and $\mathcal{X}$. We restrict our analysis for this lower bound to algorithms that are $\frac{d+1}{d+2}$-consistent, i.e.\ that achieve the standard worst-case rate regardless of the instance $(\mathcal{X},f)$. This is a convenient technical assumption as it excludes trivial or overly conservative algorithms.

\begin{lemma}\label{lemma : lower bound}
Let $\mathcal{A}$ be any $a_0$-consistent algorithm for the class of $L$-Lipschitz functions, with $a_0=\frac{d+1}{d+2}$. Then there exists a constant $c>0$ such that
\[
\mathbb{E}\big[R_T(\mathcal{A},\mathcal{X},f)\big]
\ \ge\ c\int_{\mathcal{X}\setminus \mathcal{X}_{2^{-k_T}}} \frac{dx}{\Delta(x)^{d+1}}.
\]
where $k_T := \lfloor \log_2(T)/(d+2) \rfloor$. Here, $R_T(\mathcal{A},\mathcal{X},f)$ denotes the pseudo regret incurred by algorithm $\mathcal{A}$ on the bandit instance $(\mathcal{X},f)$ over horizon $T$. In particular, this matches the outer part of the truncated integral in
Theorem~\ref{theorem:instance dependant integral regret bounds}, up to logarithmic factors, for large gaps down to the worst-case resolution scale.
\end{lemma}

The full proof of \autoref{lemma : lower bound} is deferred to \autoref{appendix : further bandits results}.

\subsection{Asymptotic bounds}\label{sec:asymptotic dim opti}
As in finite-armed bandits, instance-dependent bounds can be turned into worst-case rates by controlling the geometry of near-optimal sets. We sketch how Theorem~\ref{theorem:instance dependant integral regret bounds} yields worst-case and zooming-type bounds, and we highlight the improvement that arises when the set of maximizers $\mathcal{X}^\star$ is non-trivial. Let us denote by $d^\star$ the dimension of the maximizing set $\mathcal{X}^\star$, defined as the smallest $d^\star\ge 0$ such that there exists $c>0$ with $\mathcal{M} \left ( \mathcal{X}^\star, r \right ) \leq c\, r^{-d^\star}$ for all $r>0$.

\begin{corollary}\label{lemma:asymptotic}
The regret incurred by Algorithm~\ref{algorithm: high level} with Algorithm~\ref{algo : successive elimination} on any Lipschitz bandit instance $(\mathcal{X},f)$ satisfies
\[
R_T \;\le\; \tilde{\mathcal{O}}\!\left(T^{\frac{d_z+1}{\max(d_z,d^\star)+2}}\right),
\]
under the conditions stated in Section~\ref{sec:asymptotic dim opti}.
\end{corollary}

\begin{proof}[Proof sketch]
    This sketch provides the main insight into how the exponents in the regret appear from the computations. For a complete exposition of the proof steps, see \autoref{appendix : further bandits results}. 
    
    For a worst-case instance with zooming dimension $d_z$ and maximizer dimension $d^\star$, the covering numbers satisfy $\mathcal{M} \left ( \mathcal{X}_r\setminus \mathcal{X}_{r/2} , r \right ) = \Omega(r^{-d_z})$ and $\mathcal{M} \left ( \mathcal{X}^\star, r \right ) = \Omega(r^{- d^\star})$.
    
    The sampling-budget relation from Lemma~\ref{lemma:first regret bounds} bounds the last completed phase $k_T$ by yielding $T \geq \Omega(2^{k_T(d^\star+2)})$. Concurrently, the standard zooming analysis of the sampling budget implies $T \ge \Omega(2^{k_T(d_z+2)})$, as the number of pulls in phase $k_T$ scales as $2^{2k_T}\mathcal{M}(\mathcal{X}_{2^{-k_T}}\setminus\mathcal{X}_{2^{-(k_T-1)}}, 2^{-k_T}) \ge \Omega(2^{k_T(d_z+2)})$. Combining these gives $2^{k_T (\max(d_z, d^\star) + 2)} \le \tilde{\mathcal{O}}(T)$, which implies $2^{k_T} \le \tilde{\mathcal{O}}\!\left(T^{\frac{1}{\max(d_z, d^\star) + 2}}\right)$.

    We then decompose the truncated integral upper bound of \autoref{theorem:instance dependant integral regret bounds} into a sum over the dyadic gap scales $\mathcal{X}_{2^{-k}}\setminus \mathcal{X}_{2^{-(k+1)}}$ for $k \le k_T$. Bounding the volume of each annulus using its corresponding covering number gives:
    \begin{align*}
    R_T &\le C \log(T) \left( \sum_{k=0}^{k_T} \frac{\mathcal{M} \left ( \mathcal{X}_{2^{-k}}\setminus \mathcal{X}_{2^{-(k+1)}}, 2^{-k} \right )}{2^{-k}} + \frac{\mathcal{M} \left ( \mathcal{X}_{2^{-(k_T+1)}} \setminus \mathcal{X}^\star, 2^{-(k_T+1)} \right )}{2^{-k_T}} \right) \\
    &\le \tilde{\mathcal{O}} \left( \sum_{k=0}^{k_T} 2^{k(d_z+1)} + 2^{k_T(d_z+1)} \right) = \tilde{\mathcal{O}}\!\left(2^{k_T(d_z+1)}\right).
    \end{align*}
    Plugging in the bound on $2^{k_T}$ yields $R_T \le \tilde{\mathcal{O}}\!\left(T^{\frac{d_z+1}{\max(d_z, d^\star)+2}}\right)$.
\end{proof}

\autoref{lemma:asymptotic} shows that the range of possible asymptotic regret behaviour is broader than what was described in the literature up to now. While the above corollary definitely widens the range of powers of $T$ the regret can behave as, it is not clear from it alone what exponents are or are not possible.

We provide below a more precise characterization of the possible exponent by bounding how the zooming dimension and the dimension of $\mathcal{X}^\star$ are related, under mild geometric conditions. 

\begin{lemma}\label{lemma : relating zooming and optimal dimension}
    Let $\mathcal{X},f$ be a Lipschitz bandit instance, and $d_z,d^\star$ be the corresponding zooming dimension and optimal dimension. Assume that $\mathcal{X}^\star$ is connected and $\exists r > 0, \mathcal{B}(\mathcal{X}^\star,r) \subset \mathcal{X}$, then we have the following inequality: \[d_z +1 \geq d^\star\]

\end{lemma}

\begin{proof}[Proof insight]
    Here, we detail the main insights from the proof. The main idea is to characterize the covering number of $\mathcal{X}_r$ at scale $r$. 
    
    To that end, we first focus on $\mathcal{X} \setminus \mathcal{X}^\star \supset \cup_{k\geq 1} \mathcal{X}_{r/2^{k-1}} \setminus \mathcal{X}_{r/2^k}$. The doubling properties of covering numbers in $\mathbb{R}^d$ ensure that $\mathcal{M} \left (\mathcal{X}_r \setminus \mathcal{X}^\star , r \right ) \leq \mathcal{O} \left( r^{-d_z} \right)$. It is now sufficient to consider any point $y \in \mathcal{X}^\star$. We can use our assumptions that $\mathcal{X}^\star$ is connected and $\exists r \geq 0, \mathcal{B}(\mathcal{X}^\star,r) \subset \mathcal{X}$ to show that the affine hull of $\mathcal{X}_r\setminus \mathcal{X}^\star$ contains $\mathcal{X}^\star$. 
    
    Then, starting from an optimal covering of $\mathcal{X}_r \setminus \mathcal{X}^\star$ at scale $r$, we construct a valid covering of this affine hull by adding points along the segments between each covering point and $y$ (adding at most $\mathcal{O}(1/r)$ points per segment). Since this yields a covering of $\mathcal{X}^\star$ at scale $r$, its size is bounded below by $\mathcal{M}(\mathcal{X}^\star, r)$. Therefore, we obtain $\mathcal{M}(\mathcal{X}^\star, r) \le \mathcal{O}(1/r) \cdot \mathcal{M}(\mathcal{X}_r \setminus \mathcal{X}^\star, r)$, which implies $\mathcal{O}(r^{-d^\star}) \le \mathcal{O}(r^{-(d_z+1)})$ and directly yields the bound $d_z+1 \geq d^\star$.
    
    The full proof of \autoref{lemma : relating zooming and optimal dimension} is deferred to \autoref{appendix : further bandits results}.
\end{proof}

\section{Lipschitz Experts problem}

This section complements the results of Sections \ref{sec : introduction} and \ref{sec:algo and inst bounds} by considering the full-information analogue of stochastic Lipschitz bandits, often referred to as \emph{Lipschitz experts}. As in the bandit setting, our goal is to obtain regret guarantees that are instance-dependent, i.e., that depend explicitly on the problem instance and the horizon $T$. We also derive asymptotic zooming-dimension-type bounds and highlight how the geometry of the maximizer set $\mathcal{X}^\star$ can refine the resulting worst-case rates.

\subsection{Setting}

We adopt the \emph{uniformly Lipschitz experts} model (cf.\ \cite[Section~8.2]{kleinberg2019bandits}). A problem instance consists of a metric action space $(\mathcal{X},d_\infty)$ with $\mathcal{X}\subset[0,1]^d$, and a probability distribution $P$ over reward functions $g:\mathcal{X}\to[0,1]$ such that every $g$ in the support of $P$ fulfills Assumption~\ref{assump:Lipschitz}.

At each round $t=1,\ldots,T$:
(i) Nature draws a reward function $f_t \sim P$ i.i.d.;
(ii) the learner selects an action $x_t\in\mathcal{X}$;
(iii) the learner observes the entire function $f_t(\cdot)$ (full information) and receives reward $f_t(x_t)$.

Define the mean reward function
\[
f(x) := \mathbb{E}_{f_t\sim P}[f_t(x)], \qquad x\in\mathcal{X},
\]
and let $f^\star := \sup_{x\in\mathcal{X}} f(x)$.
Since each $f_t$ is $L$-Lipschitz, the mean function $f$ is also $L$-Lipschitz.
We measure performance by the expected cumulative regret with respect to the best action for the mean function:
\[
R_T := \mathbb{E}\!\left[\sum_{t=1}^T \bigl(f^\star - f(x_t)\bigr)\right].
\]

\subsection{Algorithm}

The full-information feedback makes the problem simpler than the bandit setting: observing $f_t(\cdot)$ means that the information acquired does not depend on the chosen action $x_t$. A natural approach is therefore to estimate the mean function $f$ by empirical averaging and to choose to play greedily with respect to this estimate. Instead of choosing greedily, our approach is to sample uniformly among the points with "reasonably good estimates". This approach ensures that our algorithm is able to leverage the fact that the maximizing sets $\mathcal{X}^\star$ can be "large" in comparison to the overall action space, when $d^\star\geq d_z$. 

We present in Algorithm~\ref{algorithm: full info} a relatively simple algorithm, which applies the above principle.

\begin{algorithm}
    \caption{Sequential Optimism with Uniform Sampling (SOUS)}
    \label{algorithm: full info}
    \begin{algorithmic}[1]
    \State \textbf{Input:} $\mathcal{X}, \delta$
    \State \textbf{Output:} $x_1, \dots, x_T$
    \State \textbf{Initialize:} $\hat{f}_0 := 0$
    \For{$t = 1, \dots, T$}
        \If{$t = 1$}
            \State Set $\mathcal{A}_1 := \mathcal{X}$
        \Else
            \State Let $\delta_{t-1} = \frac{6\delta}{\pi^2 (t-1)^2}$
            \State Set $\epsilon_{t-1} = 2L \sqrt{\frac{2d \log((t-1) L) + \log \left ( 1/\delta_{t-1} \right )}{t-1}}$
            \State Set $\mathcal{A}_t:=\{x \in \mathcal{X} \mid \hat{f}_{t-1}(x) \geq \max \hat{f}_{t-1} -\epsilon_{t-1} \}$
        \EndIf
        \State Play $x_t$, sampled uniformly in $\mathcal{A}_t$
        \State Observe $f_t$ and set $\hat{f}_t := \frac{1}{t}\sum_{k=1}^t f_k$
    \EndFor
    \end{algorithmic}
\end{algorithm}

\begin{remark}
    In a continuous domain, computing $\widehat{f}_t$, determining $\mathcal{A}_t$, and sampling uniformly from it may be computationally intractable. In practice, one can discretize $\mathcal{X}$ at a sufficiently fine scale (e.g., using a uniform grid of $O(\varepsilon_T^{-d})$ points) while only slightly inflating the regret; the analysis then applies with standard integral approximations. Alternatively, the discretization oracles from Section~\ref{sec:algo and inst bounds} can be adapted.
\end{remark}

\subsection{Regret guarantees}

As in the bandit case, we state an instance-dependent bound that separates (i) the cost of learning to rule out clearly suboptimal regions and (ii) the residual regret from the near-optimal region at the horizon-dependent resolution.

\begin{restatable}{theorem}{InstanceDependentUniformlyExperts}\label{thm:full-info-instance}
Consider a uniformly Lipschitz experts instance $(\mathcal{X},d_\infty,P)$ with mean function $f$.
Assuming a boundary regularity condition on the level sets (akin to \autoref{assump: nice opti sets}), there exists a choice of confidence level $\delta=\delta(T)$ (e.g., $\delta=T^{-3}$) such that Algorithm~\ref{algorithm: full info} satisfies
\begin{equation}
\label{eq:regret full info instance dependant}
R_T
\;\le\;
C\,\log(T) \left (\int_{\mathcal{X}\setminus \mathcal{X}_{\epsilon_T}}
\frac{1}{\Delta(x) \text{vol}(\mathcal{X}_{\Delta(x)/2})}\,dx
\;+\;
\sum_{t=1}^T \frac{1}{\text{vol}(\mathcal{X}_{\epsilon_t})} \int_{\mathcal{X}_{\epsilon_T}} \Delta(x) dx \right ),
\end{equation}
where the integral is with respect to Lebesgue measure on $[0,1]^d$ restricted to $\mathcal{X}$,
$C>0$ is a universal constant, and $\epsilon_t := \tilde{\mathcal{O}}(t^{-1/2})$ for all $t$.
\end{restatable}

\noindent
As in \autoref{theorem:instance dependant integral regret bounds}, the regret can be written as a truncated integral, and the same interpretation as in Remark~\ref{remark : regret two part} applies. 

\begin{proof}[Proof sketch]
    The proof of \autoref{thm:full-info-instance} leverages the same ideas as those of \autoref{theorem:instance dependant integral regret bounds}. Yet, it is more straightforward: in the full-information case, there is no need to use pointwise estimates of $f$, which allows us to directly sample from $\mathcal{A}_t$, making the link with integral-type bounds straightforward. The full proof is detailed in \autoref{app:proof-thm-expert}. 
\end{proof}

As for the bandit case, we can use the above \autoref{thm:full-info-instance} in order to obtain improved zooming dimension-based bounds that account for the possibility of non-trivial $\mathcal{X}^\star$. We state it below: 
\begin{corollary}
\label{lem:full-info-zooming}
Under the same assumptions, one can upper bound $R_T$ in terms of the zooming dimension $d_z$ (and, in refined forms, the intrinsic dimension $d^\star$ of $\mathcal{X}^\star$), yielding a worst-case rate of the form
\[
R_T \;\le\; \tilde{\mathcal{O}}\!\bigl(T^{\frac{1+ d_z - \max\left(d^\star,d_z\right)}{2}}\bigr).
\]
\end{corollary}

\begin{proof}[Proof sketch]
    The bound is obtained by evaluating both terms of \autoref{thm:full-info-instance} over the instance geometry. As in \autoref{lemma:asymptotic}, by using the definition of $d_z$ and $d^\star$ we can lower bound $\text{vol}(\mathcal{X}_{r})$ by $c' r^{d-\max \left ( d_z,d^\star\right )} $. 
    We can then directly bound the second term of the regret bound by $\tilde{\mathcal{O}} \left ( T^{\frac{1+ d_z-\max\left ( d_z,d^\star \right )}{2}}\right )  $. 
    For the first term, cutting the integral in dyadic annuli $\mathcal{X}_{2^{-j}} \setminus \mathcal{X}_{2^{-(j+1)}}$, and then bounding the volume of each annulus by its covering number via the zooming dimension allows us to get $\tilde{\mathcal{O}} \left ( T^{\frac{1+ d_z-\max\left ( d_z,d^\star \right )}{2}}\right )$.
\end{proof}

\begin{remark}
Note that, as long as $d^\star\leq d_z$, the regret bounds match those expected from the literature. However, as soon as $d_z < d^\star$, we gain a polynomial factor on the regret. Note that thanks to \autoref{lemma : relating zooming and optimal dimension}, we have $d^\star \le d_z+1$, which ensures that the exponent $\frac{1+d_z-\max(d_z,d^\star)}{2}$ remains non-negative, so the bound never predicts a regret that decreases with~$T$.
\end{remark}

\section{Conclusion}
We developed non-asymptotic, instance-dependent regret bounds for Lipschitz bandits, expressed as truncated integrals that capture the geometry of near-optimal regions. The analysis also yields almost matching lower bounds for the outer integral term, up to logarithmic factors, and refines classical zooming-type rates when the maximizer set has nontrivial dimension.

Two main possible extensions of these results come to mind. First, as we mention in our discussion, our analysis cannot be directly applied to other algorithms for the Lipschitz bandit problem such as the zooming algorithm or other tree-based exploration techniques. Ensuring that the instance-dependent bounds we introduced in this work are also valid for these algorithms is a promising direction. This would be especially useful as some of these algorithms are implemented and in use in practical applications.

Second, it would be interesting to see how these more precise bounds can be incorporated into bandit settings that are derived from Lipschitz bandits (such as non-stationary Lipschitz bandits) and what this increased precision in the bounds can bring to these.

Another natural direction for future work is to extend these techniques to more structured continuous-armed settings, such as convex or unimodal bandits, and to develop analogous instance-dependent bounds that exploit additional geometry.

\section*{Acknowledgements}
This work was supported in part by the French National Research Agency (ANR) through grants ANR-20-CE23-0007 and ANR-23-CE23-0002, and through the PEPR IA FOUNDRY project (ANR23-PEIA-0003). It was also supported by Hi! PARIS and the ANR/France 2030 program (ANR23-IACL-0005).

\nocite{*}
\newpage

\printbibliography
\newpage

\appendix

% (no filepath)
\appendix

\section*{Appendix Overview}
\label{app:overview}

This appendix collects the full proofs of all results stated in the main text, together with supporting lemmas and background material. We describe the contents of each appendix section below to help the reader navigate the material.

\begin{itemize}
  \item \textbf{Appendix~\ref{app:concentration-and-lemmas}: Concentration lemmas and elimination analysis.}
    This section establishes the probabilistic backbone of our analysis for the bandit setting.
    It provides anytime confidence sequences (Lemmas~\ref{lem:cs-one-arm} and~\ref{lem:cs-one-phase}),
    defines the global good event $\mathcal{E}$ (Definition~\ref{def:global-good}), and proves the key
    structural lemmas: near-optimal points are never eliminated (Lemma~\ref{lem:nearopt-never-elim}),
    the optimal set stays active (Lemma~\ref{lem:opt-stays-active-complete}), gap control in the next
    active region (Lemma~\ref{lem:gap-next-active-complete}), and per-arm pull bounds
    (Lemma~\ref{lem:pull-bound-complete}).

  \item \textbf{Appendix~\ref{app:proof-first-lemma-complete}: Regret derivation — proof of Lemma~\ref{lemma:first regret bounds}.}
    Using the concentration and elimination lemmas of the previous section, this section carries out
    the full regret derivation, yielding the packing-number-based bound of Lemma~\ref{lemma:first regret bounds}
    together with the characterization of $k_T$.

  \item \textbf{Appendix~\ref{app: proof theorem 1}: Proof of Theorem~\ref{theorem:instance dependant integral regret bounds} (integral regret bound).}
    This section bridges the packing-number bound and the integral form stated in the main theorem.
    It introduces the series-to-integral lemma (Lemma~\ref{lemma : sum integral eq}) and carries out
    the formal proof of Theorem~\ref{theorem:instance dependant integral regret bounds}.
    It also establishes the accompanying lower bound (Lemma~\ref{lemma : lower bound}) and the relation
    between zooming and optimal dimensions (Lemma~\ref{lemma : relating zooming and optimal dimension}).

  \item \textbf{Appendix~\ref{sec:discussion-expert}: Discussion — Lipschitz expert modeling.}
    This section discusses the modeling choices for the full-information (expert) setting,
    clarifying the distinction between the Lipschitz expert and uniformly Lipschitz expert models.

  \item \textbf{Appendix~\ref{sec:expert-proofs}: Expert setting background and proofs.}
    This section contains the concentration results specific to the full-information setting
    (Lemma~\ref{lem:uniform_concentration}) and the full proofs of the expert regret theorem
    (Theorem~\ref{thm:full-info-instance}, in Appendix~\ref{app:proof-thm-expert})
    and its zooming-dimension corollary (Corollary~\ref{lem:full-info-zooming}, in Appendix~\ref{proof: corollary full info}).

  \item \textbf{Appendix~\ref{sec : relaxed}: The one-sided Lipschitz setting.}
    This section describes the modifications required to handle the one-sided Lipschitz assumption
    (Assumption~\ref{assump : Lipschitz one sided}), including the adapted algorithm
    (Algorithm~\ref{algorithm: high level one sided}), the structural propositions that remain
    valid, and the positive-gap regime result (Theorem~\ref{thm:onesided-positive-gap}).

  \item \textbf{Appendix~\ref{app : restated from litt}: Results restated from the literature.}
    This section restates the key lower-bound building blocks from \cite{shekhar2022instance}
    (Definition~\ref{def:lipschitz-complexity} and Proposition~\ref{prop:lipschitz-regret})
    in the notation of the present paper.
\end{itemize}

\bigskip

\section{Concentration lemmas and elimination analysis}
\label{app:concentration-and-lemmas}

This section gives the probabilistic and structural lemmas used in the proof of Lemma~\ref{lemma:first regret bounds}.
We work with $\|\cdot\|_\infty$ throughout and assume the reward noise is conditionally $1$-sub-Gaussian.

\paragraph{Normalization.}
To avoid carrying Lipschitz constants, we first present the proof for $L=1$.
For general $L\ge l$, replace each radius $r$ by $r/L$ in the metric arguments (equivalently, analyze $f/L$).

% ============================================================
\subsection{ Concentration lemmas}

Fix a phase $k$ and an arm $a\in S_k$. Let $\widehat\mu_{k,\ell}(a)$ be the empirical mean of $a$
after it has been pulled $\ell$ times \emph{during phase $k$} (such as at the end of round $\ell$).

\begin{lemma}[Anytime confidence sequence for one arm]\label{lem:cs-one-arm}
Let $(\xi_i)_{i\ge1}$ be conditionally $1$-sub-Gaussian martingale differences.
For any $\delta\in(0,1)$ define
\[
u(\ell,\delta)\ :=\ \sqrt{\frac{2}{\ell}\log\!\Big(\frac{\pi^2\ell^2}{6\delta}\Big)}\,,\qquad \ell\ge 1.
\]
Then for any fixed arm $a$ (and any fixed phase index),
\[
\Proba \!\Big(\exists \ell\ge 1:\ \big|\widehat\mu_{k,\ell}(a)-f(a)\big|>u(\ell,\delta)\Big)\ \le\ \delta.
\]
\end{lemma}

\begin{proof}
For each fixed $\ell\ge 1$, conditional sub-Gaussianity and Hoeffding's inequality yield
\[
\Proba \!\Big(\big|\widehat\mu_{k,\ell}(a)-f(a)\big|>\sqrt{\tfrac{2}{\ell}\log(1/\delta_\ell)}\Big)\ \le\ \delta_\ell.
\]
Choose $\delta_\ell:=\frac{6\delta}{\pi^2\ell^2}$ so that $\sum_{\ell\ge 1}\delta_\ell=\delta$.
With this choice, the deviation threshold equals $u(\ell,\delta)$.
A union bound over $\ell\ge 1$ gives the claim.
\end{proof}

\begin{lemma}[Anytime concentration within one phase]\label{lem:cs-one-phase}
Fix a phase $k$ and its discretization $S_k$ with $n_k:=|S_k|$.
For any $\delta_k\in(0,1)$, the following inequality holds:
\[ \Proba \left ( \exists l\geq 1, \exists a \in S_k, \big|\widehat\mu_{k,\ell}(a)-f(a)\big|
\ > \
u\!\Big(\ell,\frac{\delta_k}{n_k}\Big) \right ) \leq \delta_k
\]
\end{lemma}

\begin{proof}
Applying Lemma~\ref{lem:cs-one-arm} to each $a\in S_k$ with confidence level $\delta_k/n_k$, and then doing a union bound provides the desired result. 
\end{proof}

\begin{definition}[Global good event $\mathcal{E}$]\label{def:global-good}
Let $(\delta_k)_{k\ge 1}$ satisfy $\sum_{k\ge 1}\delta_k\le \delta$ (e.g.\ $\delta_k=\frac{6\delta}{\pi^2k^2}$).
Let $\mathcal{E}$ be the event that Lemma~\ref{lem:cs-one-phase} holds for every phase $k\ge 1$.
\end{definition}

Notice that by applying \autoref{lem:cs-one-phase} to each phase and then doing a union bound, it is straightforward that $\Proba(\mathcal{E})\ge 1-\delta$.

In the remainder, we condition on $\mathcal{E}$.

% ============================================================
\subsection{ Successive elimination and invariants}

Recall that in phase $k$ PACO uses accuracy $r_k:=2^{-k}$ and the elimination rule 
(from Algorithm~\ref{algo : successive elimination}) yields:
\[
A_{\ell+1}\ :=\
\Bigl\{a\in A_\ell:\ \widehat\mu_{k,\ell}(a)+u_{k,\ell}\ \ge\
\max_{j\in A_\ell}\widehat\mu_{k,\ell}(j)-u_{k,\ell}\ -\ r_k\Bigr\},
\]
where
\[
u_{k,\ell}\ :=\ \sqrt{\frac{2}{\ell}\log\!\Big(\frac{\pi^2\,\ell^2\,|S_k|}{6\,\delta_k}\Big)}.
\]
The phase stops at the first $\ell_k$ such that $u_{k,\ell_k}\le r_k/4$ and outputs $\widehat S_k:=A_{\ell_k}$.
The next active region is
\[
\mathcal{A}_{k+1}\ :=\ \bigcup_{a\in \widehat S_k} B_\infty(a,r_k).
\]

\begin{lemma}[A near-optimal discretization point is never eliminated]\label{lem:nearopt-never-elim}
Condition on $\mathcal{E}$.
Fix phase $k$ and let $a\in S_k$ such that $\Delta(a)\le r_k$.
Then $a\in A_\ell$ for all $\ell \geq 1$, and therefore $a\in\widehat S_k$.
\end{lemma}

\begin{proof}
Fix $\ell\ge 1$ and suppose $a\in A_\ell$.
Let $m_\ell\in\arg\max_{j\in A_\ell}\widehat\mu_{k,\ell}(j)$.
On $\mathcal{E}$ we have $\widehat\mu_{k,\ell}(a)\ge f(a)-u_{k,\ell}\ge f^\star-r_k-u_{k,\ell}$.
Also for any $j\in A_\ell$, $\widehat\mu_{k,\ell}(j)\le f(j)+u_{k,\ell}\le f^\star+u_{k,\ell}$, so
$\max_{j\in A_\ell}\widehat\mu_{k,\ell}(j)-u_{k,\ell}\le f^\star$.
Therefore
\[
\widehat\mu_{k,\ell}(a)+u_{k,\ell}
\ \ge\ f^\star-r_k
\ \ge\ \max_{j\in A_\ell}\widehat\mu_{k,\ell}(j)-u_{k,\ell}-r_k,
\]
so $a$ satisfies the retention condition and belongs to $A_{\ell+1}$.
By induction over $\ell$, $a$ is never eliminated and thus is in $\widehat S_k$.
\end{proof}

\begin{lemma}[Optimal set remains active]\label{lem:opt-stays-active-complete}
Condition on $\mathcal{E}$.
Then for every phase $k\le k_T$,
\[
\mathcal{X}^\star\subseteq \mathcal{A}_k.
\]
\end{lemma}

\begin{proof}
We prove the above statement by induction on $k$. For $k=1$, the statement is trivially true as $\mathcal{A}_1=\mathcal{X}$.

Let $1 \leq k < k_T$, and assume $\mathcal{X}^\star\subseteq \mathcal{A}_k$.

Let $x^\star\in\mathcal{X}^\star$.
By the oracle covering property (Assumption~\ref{assump:good oracle}), there exists $a\in S_k\subseteq\mathcal{A}_k$
such that $\|a-x^\star\|_\infty\le \frac{r_k}{L}$.

The Lipschitzness of $f$ (Assumption~\autoref{assump:Lipschitz}) then yields,
\[
f(a)\ \ge\ f(x^\star)-L \|a-x^\star\|_\infty\ \ge\ f^\star-r_k,
\]
By definition of $\Delta$, we obtain $\Delta(a)\le r_k.$

Lemma~\ref{lem:nearopt-never-elim} then implies that $a$ \textit{survives}, i.e. $a\in\widehat S_k$.
Therefore, since $x^\star\in B_\infty(a,r_k/L)\subseteq \mathcal{A}_{k+1}$, we have $ x^\star \in \mathcal{A}_{k+1}$.

Since $x^\star$ was arbitrary, $\mathcal{X}^\star\subseteq \mathcal{A}_{k+1}$ which concludes the proof. 
\end{proof}

\begin{lemma}[Gap control in the next active region]\label{lem:gap-next-active-complete}
Condition on $\mathcal{E}$.
For every phase $k\le k_T-1$ and every $x\in\mathcal{A}_{k+1}$,
\[
\Delta(x)\ \le\ 4r_k.
\]
\end{lemma}

\begin{proof}

By Lemma~\ref{lem:opt-stays-active-complete}, $\mathcal{X}^\star\subseteq \mathcal{A}_k$. Hence there exists $(x^\star,z) \in \mathcal{X}^\star\times S_k$ such that $\|x^\star-z\|_\infty \leq \frac{r_k}{L}$, which by Lipschitzness of $f$ implies $\Delta(z) \leq r_k$. 

Let $x^{\mathrm{best}}\in\arg\max_{z\in S_k} f(z)$. By the above statement, we have $\Delta(x^{\mathrm{best}})\le r_k$.

For any $x\in\mathcal{A}_{k+1}$, by definition of $\mathcal{A}_{k+1}$ there exists $y\in \widehat S_k$
such that $\|x-y\|_\infty\le \frac{r_k}{L}$. By Lipschitzness of $f$,
\[
\Delta(x)=f^\star-f(x)\ \le\ f^\star-f(y)+L \|x-y\|_\infty\ =\ \Delta(y)+r_k.
\]

Therefore, at the final phase $\ell_k$ we have $u_{k,\ell_k}\le r_k/4$.
Since $y\in\widehat S_k=A_{\ell_k}$, it was not eliminated, so the conservative rule implies
\[
\widehat\mu_{k,\ell_k}(y)+u_{k,\ell_k}
\ \ge\
\max_{j\in A_{\ell_k}}\widehat\mu_{k,\ell_k}(j)-u_{k,\ell_k}-r_k
\ \ge\
\widehat\mu_{k,\ell_k}(x^{\mathrm{best}})-u_{k,\ell_k}-r_k,
\]
where the last step uses $x^{\mathrm{best}}\in A_{\ell_k}$ (it is never eliminated because $\Delta(x^{\mathrm{best}})\le r_k$,
by Lemma~\ref{lem:nearopt-never-elim}).

On $\mathcal{E}$ we have $\widehat\mu_{k,\ell_k}(y)\le f(y)+u_{k,\ell_k}$
and $\widehat\mu_{k,\ell_k}(x^{\mathrm{best}})\ge f(x^{\mathrm{best}})-u_{k,\ell_k}$, hence
\[
f(y)+2u_{k,\ell_k}\ \ge\ f(x^{\mathrm{best}})-2u_{k,\ell_k}-r_k
\quad\Rightarrow\quad
f(y)\ \ge\ f(x^{\mathrm{best}})-4u_{k,\ell_k}-r_k.
\]
Thus
\[
\Delta(y)=f^\star-f(y)
\ \le\ (f^\star-f(x^{\mathrm{best}})) + 4u_{k,\ell_k}+r_k
\ \le\ r_k + 4(r_k/4) + r_k\ =\ 3r_k.
\]
Plugging into $\Delta(x)\le \Delta(y)+r_k$ yields $\Delta(x)\le 4r_k$.
\end{proof}

\subsection{Pull-count bound within a phase}

We now move to the following standard step in any stochastic regret proof: upper-bounding how many times an arm is pulled. In our case, we will bound, within a phase $k$, how many times each point in $S_k$ is played.

Let us denote $N_k(a)$ the number of times a point $a\in S_k$ is played during phase $k$.

\begin{lemma}[Per-arm pull bound]\label{lem:pull-bound-complete}
For every phase $k\le k_T$ and every $a\in S_k$,
\[
N_k(a)
\ \le\
c_0\,\frac{\log\!\big(\frac{|S_k|}{\delta_k}\big)}{\max\{\Delta(a),r_k\}^2}
\]
for a universal constant $c_0>0$.
\end{lemma}

\begin{proof}
Fix phase $k$. Algorithm~\autoref{algo : successive elimination} plays each point $a \in S_k$ at most once per round; let $\ell_k$ be the last round of phase $k$, therefore for all $a \in S_k$, $N_k(a) \leq \ell_k$. 

Since the phase ends at the first $\ell_k$ such that $u_{k,\ell_k}\le r_k/4$, we can obtain a bound on $N_k(a)$. 

Solving $u_{k,\ell}\le r_k/4$ yields
\[
\ell_k \ \le\ c\,r_k^{-2}\log\!\Big(\frac{|S_k|}{\delta_k}\Big)
\]
for a universal $c$ (the extra $\log(\ell)$ term is absorbed into the constant by standard bounds).

Thus $N_k(a)\le \ell_k\le c\,\log(|S_k|/\delta_k)/r_k^2$. 

For all $a \in S_k$ such that $\Delta(a)\le 3 r_k$, the above inequality matches the claim up to a constant factor.

\emph{Case $\Delta(a) > 3 r_k$}
\newline 
Let $b\in\arg\max_{z\in S_k} f(z)$.
As in the proof of Lemma~\ref{lem:gap-next-active-complete}, by oracle covering $\Delta(b)\le r_k$ and 
by Lemma~\ref{lem:nearopt-never-elim}, $b$ is never eliminated, so for any $\ell$, $b\in A_\ell$.

Let us denote by $\tau$ the first round at which
\[
\Delta(a) > 2 r_k + 4 u_{k,\tau}.
\]
By conditioning on $\mathcal{E}$, we ensure that at round $\tau$ we have
\[
\widehat\mu_{k,\tau}(a) \le f(a) + u_{k,\tau}
\]
and 
\[
\widehat\mu_{k,\tau}(b) \ge f(b) - u_{k,\tau}.
\]
Since $f(b) - f(a) = \Delta(a) - \Delta(b) \ge \Delta(a) - r_k$, our condition $\Delta(a) > 2 r_k + 4 u_{k,\tau}$ implies:
\[
f(b) - f(a) > r_k + 4 u_{k,\tau}.
\]
Combining these, we obtain:
\[
\widehat\mu_{k,\tau}(a) + u_{k,\tau} \le f(a) + 2 u_{k,\tau} < f(b) - 2 u_{k,\tau} - r_k \le \widehat\mu_{k,\tau}(b) - u_{k,\tau} - r_k.
\]
This triggers the elimination rule of Algorithm~\ref{algo : successive elimination}, meaning $a$ is eliminated at round $\tau$. Thus $N_k(a) \le \tau$.

Solving for $\tau$ such that $u_{k,\tau} < \frac{\Delta(a) - 2r_k}{4}$, and using $\Delta(a) > 6 r_k$ (which yields $\Delta(a) - 2r_k > \frac{2}{3} \Delta(a)$), we obtain:
\[
\tau\ \le\ c'\,\frac{\log(|S_k|/\delta_k)}{\Delta(a)^2},
\]
and since here $\max\{\Delta(a),r_k\}=\Delta(a)$, this is the desired bound.
\end{proof}

\section{Regret derivation}\label{app:proof-first-lemma-complete}

With the previous technical lemmas, we are fully equipped to derive Lemma~\ref{lemma:first regret bounds}.
We restate it here for completeness :

\LemmaInstanceRegret*

\begin{proof} 
The expected cumulative regret is defined as $R_T = \mathbb{E}[\sum_{t=1}^T (f^\star - f(x_t))]$. We decompose this expectation over the global good event $\mathcal{E}$ and its complement $\mathcal{E}^c$:
\begin{equation}\label{eq: expectation decomposition}
R_T \;\le\; \mathbb{E}\left[\sum_{t=1}^T \Delta(x_t)\,\mathbf{1}_{\mathcal{E}}\right] + \mathbb{P}(\mathcal{E}^c) T \max_{x \in \mathcal{X}} \Delta(x).
\end{equation}
By the union bound, $\mathbb{P}(\mathcal{E}^c) \le \delta = T^{-3}$. Under Assumption~\ref{assump:Lipschitz}, the suboptimality gap is bounded by $\max_{x \in \mathcal{X}} \Delta(x) \le L \operatorname{diam}_\infty(\mathcal{X}) \le L$ (with our normalization $L=1$). Therefore, the failure term is bounded by $L/T^2 \le o(1)$, which is a negligible constant that can be absorbed into the final term.

We now bound the cumulative pseudo-regret on the good event $\mathcal{E}$. Algorithm~\ref{algorithm: high level} works in multiple phases, each having its own level of approximation of $f$. Let $k_T$ be the active phase index containing the horizon $T$, so that $s_{k_T-1} < T \le s_{k_T}$. The total pseudo-regret incurred on $\mathcal{E}$ satisfies:
\begin{equation}\label{eq: sum phase based regret}
\sum_{t=1}^T \Delta(x_t) \;\le\; \sum_{k=1}^{k_T} R_k,
\end{equation}
where $R_k$ is the regret incurred if phase $k$ is run to completion (for the last, incomplete phase $k_T$, the actual number of pulls of any arm $a$ is bounded by $N_{k_T}(a)$, so this bound holds conservatively).

Fix a phase index $k \leq k_T$, we can decompose the regret based on how many times each point in $S_k$ is played during phase $k$ as follows :
\begin{equation}\label{eq : regret phase k, count points}
R_k\ =\ \sum_{a\in S_k} N_k(a)\Delta(a).
\end{equation}

Combining \eqref{eq : regret phase k, count points} with Lemma~\ref{lem:pull-bound-complete} yields 
\begin{equation}\label{eq : temporary instance delta}
R_k
\ \le\
c_0\log\!\Big(\frac{|S_k|}{\delta_k}\Big)
\sum_{a\in S_k}\frac{\Delta(a)}{\max\{\Delta(a),r_k\}^2}.
\end{equation}

To account more tightly for the fact that within $S_k$, $\Delta(a)$ can still vary a lot, we define the following subsets of $S_k$, based on the values of $\Delta$: $S_k^{l} := S_k \cap  (\mathcal{X}_{r_l}\setminus \mathcal{X}_{r_{l+1}})$ . We can thus rewrite \eqref{eq : temporary instance delta}
\begin{equation} \label{eq: temporary instance delta 2}
    R_k \leq c_0 \log \left ( \frac{|S_k|}{\delta_k} \right ) \sum_{k' \geq \max \left (1,k-3 \right )} \sum_{a\in S_k^{k'}} \frac{\Delta(a)}{\max\{\Delta(a),r_k\}^2}
\end{equation}
Note that it is sufficient to take the sum starting at $\max \left (1,k-3 \right )$ by \autoref{lem:gap-next-active-complete}.

We can now upper-bound the values of $\frac{\Delta(a)}{\max\{\Delta(a),r_k\}^2}$ and get:
\begin{align} \label{eq:temporary instance delta 3}
    R_k &\leq 4c_0 \log \left ( \frac{|S_k|}{\delta_k} \right ) \left [ \sum_{k'=\max(1,k-3)}^k  \frac{\left |S_k^{k'}\right |}{r_{k}}
    + \sum_{k'=k+1}^{k_T} \frac{r_{k'} \left |S_k^{k'}\right |}{r_{k}^2} + \frac{r_{k_T+1} |S_k \cap \mathcal{X}_{r_{k_T+1}}|}{r_{k}^2} \right]
\end{align}

In order to bound $\left |S_k^{k'}\right |$, we use Assumption~\ref{assump:good oracle} on our oracle. Indeed, remember that as an output of our oracle, $S_k$ is $(c_{\mathrm{sep}}r_k)$-separated. Hence for any $\mathcal{Y}\subseteq\mathcal{X}$,
\[
|S_k\cap\mathcal{Y}|
\ \le\
\mathcal{N}(\mathcal{Y},c_{\mathrm{sep}}r_k)
\ \le\
C_{\mathrm{double}} \mathcal{N}(\mathcal{Y},r_k),
\]
where the last inequality follows from the doubling property of $\mathbb{R}^d$
(since $c_{\mathrm{sep}}r_k\le r_k$).
Applying this to $S_k^{k'}$ gives
$\left |S_k^{k'}\right |\le \mathcal{N}(\mathcal{X}_{r_{k'}}\setminus \mathcal{X}_{r_{k'+1}},r_k)$.

We combine the same bounding techniques with the definition of $\mathcal{X}_{r_{k'}}$ and the fact that $r_k=2^{-k}$ on the second and third terms of the inequality to get:
\begin{equation}\label{eq : complicated but necessary}
\begin{split}
    R_k \leq 4c_0 \log \left ( \frac{|S_k|}{\delta_k} \right )
    \Bigg[
    &\sum_{k'=\max(1,k-3)}^k
    \frac{\mathcal{N}(\mathcal{X}_{r_{k'}}\setminus \mathcal{X}_{r_{k'+1}},r_k)}{r_{k}}
    + \sum_{k'=k+1}^{k_T}
    \frac{\mathcal{N}(\mathcal{X}_{r_{k'}}\setminus \mathcal{X}_{r_{k'+1}},r_k)}{2^{k'-k} r_{k}} \\
    &\qquad + 4\, \mathcal{N}\left( \mathcal{X}_{r_{k_T+1}}, r_k \right )
    \frac{r_{k_T+1}}{r_{k}^2}
    \Bigg]
\end{split}
\end{equation}

To prepare for summing the $R_k$, we compute the following \begin{align*}
    &\sum_{k=1}^{k_T} \sum_{k'=\max(1,k-3)}^k  \frac{\mathcal{N}(\mathcal{X}_{r_{k'}}\setminus \mathcal{X}_{r_{k'+1}},r_k)}{r_{k}} +\sum_{k=1}^{k_T} \sum_{k'=k+1}^{k_T} \frac{\mathcal{N}(\mathcal{X}_{r_{k'}}\setminus \mathcal{X}_{r_{k'+1}},r_k)}{2^{k'-k} r_{k}} \\
     & \leq \sum_{k'=1}^{k_T} \sum_{k=k'}^{\min(k_T,k'+2)}  \frac{\mathcal{N}(\mathcal{X}_{r_{k'}}\setminus \mathcal{X}_{r_{k'+1}},r_k)}{r_{k}} + \sum_{k=1}^{k_T}\sum_{k'=k+1}^{k_T} \frac{\mathcal{N}(\mathcal{X}_{r_{k'}}\setminus \mathcal{X}_{r_{k'+1}},r_k)}{2^{k'-k} r_{k}} \label{eq:Before packing arg} \numberthis \\
     &  \leq(1 + 2^{d+1}C_d + 4^{d+1} C_d) \sum_{k'=1}^{k_T} \frac{\mathcal{N}(\mathcal{X}_{r_{k'}}\setminus \mathcal{X}_{r_{k'+1}},r_{k'})}{r_{k'}} \label{eq :after packing arg} \numberthis
\end{align*}

We used, to get from \eqref{eq:Before packing arg} to \eqref{eq :after packing arg}, the fact that $\mathbb{R}^d$ is a doubling space of dimension $d$ and $C_d$ is the doubling constant of $\mathbb{R}^d$ (as in \cite{heinonen2001lectures} 10.13 Doubling spaces)Using the above bound and \eqref{eq : complicated but necessary}, we can now go back to \eqref{eq: sum phase based regret} and sum over all phases $k \le k_T$ to get
\begin{align}
    \sum_{k=1}^{k_T} R_k
    &\leq \log \left ( \frac{\max_k |S_k|}{\min_k \delta_k} \right )
    \Bigg[
    c_1\sum_{k=1}^{k_T}
    \frac{\mathcal{N}(\mathcal{X}_{r_{k}}\setminus \mathcal{X}_{r_{k+1}},r_{k})}{r_{k}}
    \notag \\
    &\qquad\qquad + c_2 \sum_{k=1}^{k_T} \mathcal{N}\left( \mathcal{X}_{r_{k_T+1}}, r_k \right )
    \frac{r_{k_T+1}}{r_{k}^2}
    \Bigg] \\
    &\leq \log \left ( \frac{\max_k |S_k|}{\min_k \delta_k} \right )
    \Bigg[
    c_1 \sum_{k=1}^{k_T} 2^k\mathcal{N}(\mathcal{X}_{2^{-k}}\setminus \mathcal{X}_{2^{-(k+1)}},2^{-k})
    \notag \\
    &\qquad\qquad + c_2 2^{-(k_T+1)} \sum_{k=1}^{k_T} 2^{2k}
    \mathcal{N}\left( \mathcal{X}_{2^{-(k_T+1)}}, 2^{-k} \right )
    \Bigg] \\
    &\leq \log \left ( \frac{k_T^2 \pi^2 T}{6 \delta} \right )
    \Bigg[
    c_1 \sum_{k=1}^{k_T} 2^k\mathcal{N}(\mathcal{X}_{2^{-k}}\setminus \mathcal{X}_{2^{-(k+1)}},2^{-k})
    \notag \\
    &\qquad\qquad + c_2 2^{-(k_T+1)} \sum_{k=1}^{k_T} 2^{2k}
    \mathcal{N}\left( \mathcal{X}_{2^{-(k_T+1)}}, 2^{-k} \right )
    \Bigg]
    \label{eq : the second term of regret}
    \end{align}

where $c_1=8c_0(1+2^{d+1} C_d + 4^{d+1} C_d)$ and $c_2=16c_0$.
This concludes the regret part of the proof.

Let us now characterize $k_T$. Our first two equations describe two facts: the sum of the durations of all completed phases up to $k_T-1$ is at most $T$, and each phase duration is the sum of the number of times each point in $S_k$ was pulled in that phase.
\begin{align}
    & s_{k_T-1} = \sum_{k=1}^{k_T-1} \tau_k \le T \le \sum_{k=1}^{k_T} \tau_k = s_{k_T} \label{eq: time equal sum phases} \\
    \forall k\le k_T\ ,\ & \tau_k = \sum_{a \in S_k} N_k(a) \label{eq: phases equal sum N} 
\end{align}

We provide below a first, simple approach that allows us to get an upper bound on $k_T$, which does not depend on the geometry of $\mathcal{X}^\star$ or on the level sets $\mathcal{X}_{r}$.

By \autoref{lem:nearopt-never-elim}, at phase $k$, $\mathcal{X}^\star \subseteq \mathcal{A}_k$ and by Assumption~\autoref{assump:good oracle} and because $f$ is $L$-Lipschitz, we know that at least one point of $S_k$, that we denote $x_k^{best}$, is such that $\Delta(x_k^{best}) \leq r_k$. By Lemma~\autoref{lem:nearopt-never-elim} this point is never eliminated in the phase, hence $N_k(x_k^{best}) \ge \ell_k$ where $\ell_k$ is the final round of phase $k$. By \autoref{algo : successive elimination}, $\ell_k$ is the smallest integer $i$ such that:
\begin{equation*}
    u_{k,i}\le r_k/4
\end{equation*}

We therefore have $\ell_k \ge \lceil \frac{32c \log (\frac{T}{\delta})}{r_k^2} \rceil$ hence
\begin{equation}
    \ell_k \ge 2^{2k} c'\log\left ( \frac{T}{\delta} \right )
\end{equation}

When $k_T \ge 2$, phase $k_T-1$ is fully completed. Since phase lengths get exponentially longer as $k$ grows, it is enough (up to constants) to consider only this last finished phase:
\begin{align}
    2^{2(k_T-2)} \le \ell_{k_T-1} \le N_{k_T-1}(x_{k_T-1}^{best}) \le \tau_{k_T-1} \le T \\
    k_T \le 2 + \frac{1}{2} \log_2 (T)
\end{align}

In order to obtain a more precise bound, one should focus on refining the bound $N_k(x_k^{best}) \le \tau_k$ and account for the fact that the number of points which are guaranteed to not be eliminated during phase $k$ is at least $c_{\mathrm{sep}} \mathcal{M} \left ( \mathcal{X}_{r_{k+1}}, r_k/L \right )$. The above argument can then be applied to each of these points as it was to $x_k^{best}$, which yields the desired bound linking $k_T$ and $T$:
\begin{equation*}
    T \ge \tau_{k_T-1} \ge C_3 2^{2(k_T-2)} \mathcal{M} \left ( \mathcal{X}_{r_{k_T}}, 2^{-(k_T-1)}/L \right ) \log \left ( \frac{T}{\delta} \right )
\end{equation*}
where $C_3 = c_{\mathrm{sep}}c'$. Since $\mathcal{X}^\star \subseteq \mathcal{X}_{r_{k_T}}$ and by monotonicity of the covering number with respect to the radius, we naturally deduce the bound presented in Lemma~\ref{lemma:first regret bounds}.

\end{proof}

\subsection{Theorem 1:  Instance-dependent bounds as integrals}\label{app: proof theorem 1}

In this section, we provide the necessary steps to obtain instance-dependent regret bounds in integral forms as in \autoref{theorem:instance dependant integral regret bounds}.

The proof of \autoref{theorem:instance dependant integral regret bounds} builds on \autoref{lemma:first regret bounds} as well as on the following lemmas. These are essentially series integral equivalence lemmas; in essence they allow us to show that, despite our use of discretization and of an oracle $\mathbf{O}$, up to constants, the regret still scales as in the continuous problem.

To be able to link packing number-based regret bounds to integrals, we require the action sets and the optimal sets to be somewhat well behaved. These are mild geometric assumptions. At a high level, we only require that a small fraction of each ball centered around any close-to-optimal point be included in $\mathcal{X}$. Such assumptions are common in Lipschitz function analysis \citep{bachoc2021instance,hu2020smooth}.

Let us denote, for a set $\mathcal{Y} \subset \mathbb{R}^d$, and a radius $r>0$, the set $\mathcal{B}\left (\mathcal{Y},r\right )$ of all points at distance at most $r$ from $\mathcal{Y}$. We also denote $v_r$ the volume of an $\ell_\infty$ ball of radius $r$.

\begin{assumption}\label{assump: nice opti sets}
There exists $l \in \mathbb{N}$ and $\gamma \in (0,1]$ such that for all $k\geq l$ and all $x\in \mathcal{X}_{2^{-k}}$,
\[
\mathrm{vol}\left(B_\infty\left(x,\frac{2^{-(k+3)}}{L}\right)\cap \mathcal{X}\right)\ \ge\ \gamma\, v_{2^{-(k+3)}/L}.
\]
\end{assumption}

\begin{remark}
\autoref{assump: nice opti sets} uses radius $2^{-(k+3)}/L$, while its counterpart in the main text (\autoref{assump: nice opti sets main}) is stated with radius $2^{-k}/L$. These two formulations are equivalent up to universal constants, since $\mathcal{X}\subset[0,1]^d$ is a doubling metric space and the volumetric condition at any scale $r$ implies the same condition at scale $cr$ (for any fixed constant $c>0$) with a modified constant $\gamma$.
\end{remark}

\begin{lemma} [Series Integral inequality] \label{lemma : sum integral eq}
    Under \autoref{assump: nice opti sets} there exists a constant $C>0$ (depending on $d,\gamma,L$) such that
    \[
    \sum_{k=l}^{k_T-1} 2^{k} \mathcal{N}\left( \mathcal{X}_{2^{-k}} \setminus \mathcal{X}_{2^{-(k+1)}}, \frac{2^{-k}}{L}\right )
    \ \leq\ C\int_{\mathcal{X}_{2^{-l}} \setminus \mathcal{X}_{2^{-k_T}}} \frac{1}{\left ( f^\star - f(x) \right )^{d+1}} dx.
    \]
\end{lemma}

\begin{proof}
Recall that by definition, $\forall x \in \mathcal{X}_{2^{-k}} \setminus \mathcal{X}_{2^{-(k+1)}}, 2^{-k} \geq \Delta(x) \geq 2^{-(k+1)}$.
    \begin{align}
        \int_{\mathcal{X}_{2^{-l}} \setminus \mathcal{X}_{2^{-k_T}}} \frac{1}{\left ( \Delta(x) \right )^{d+1}} dx & \geq \sum_{k=l}^{k_T-1}  \int_{\mathcal{X}_{2^{-k}} \setminus \mathcal{X}_{2^{-(k+1)}}} \frac{1}{\left ( \Delta(x) \right )^{d+1}} dx \\
        & \geq \sum_{k=l}^{k_T-1} \frac{1}{2^{-(d+1)(k+1)}} \int_{\mathcal{X}_{2^{-k}} \setminus \mathcal{X}_{2^{-(k+1)}}} dx \\
        & \geq \frac{1}{2^{3d}} \sum_{k=l+1}^{k_T-1} 2^{(d+1)(k+1)}  \int_{\mathcal{X}_{2^{-(k-1)}} \setminus \mathcal{X}_{2^{-(k+2)}}} dx \\
        & \geq \frac{1}{2^{3d}}  \sum_{k=l+1}^{k_T-1} 2^{(d+1)(k+1)}  vol (\mathcal{X}_{2^{-(k-1)}} \setminus \mathcal{X}_{2^{-(k+2)}})\end{align}

We can then bound, for $k\geq l$, $vol(\mathcal{X}_{2^{-(k-1)}} \setminus \mathcal{X}_{2^{-(k+2)}} )$, as follows. Let
$(x_i)_{i \in \mathcal{N} \left ( \mathcal{X}_{2^{-k}} \setminus \mathcal{X}_{2^{-(k+1)}} , 2^{-(k+2)}/L \right )}$ be a packing for
$\mathcal{X}_{2^{-k}} \setminus \mathcal{X}_{2^{-(k+1)}}$.
By definition, these points are at distance at least $2^{-(k+2)}/L$ from one another, so the balls
$B_\infty\left (x_i,2^{-(k+3)}/L \right )$ are disjoint. Furthermore, since $f$ is $L$-Lipschitz, for all $x_i$, we know that for any point $y \in B_\infty(x_i, \frac{2^{-(k+3)}}{L})$ we have
$2^{-(k+1)} - 2^{-(k+3)}  \leq \Delta(y)\leq  2^{-k} + 2^{-(k+3)} $, which ensures that they belong to $\mathcal{X}_{2^{-(k-1)}} \setminus \mathcal{X}_{2^{-(k+2)}}$. Hence we can write:  
\begin{equation*}
vol(\mathcal{X}_{2^{-(k-1)}} \setminus \mathcal{X}_{2^{-(k+2)}} ) \geq \sum_{(x_i)_{i \in \mathcal{N} \left ( \mathcal{X}_{2^{-k}} \setminus \mathcal{X}_{2^{-(k+1)}} , 2^{-(k+2)}/L \right )}} vol \left (B_\infty\left (x_i,\frac{2^{-(k+3)}}{L} \right ) \cap \mathcal{X} \right ).
\end{equation*}

This then yields the following inequality and allows us to use \autoref{assump: nice opti sets} :

        \begin{align}
        \int_{\mathcal{X}_{2^{-l}} \setminus \mathcal{X}_{2^{-k_T}}} \frac{1}{\left ( \Delta(x) \right )^{d+1}} dx
        & \geq \frac{1}{2^{3d}}  \sum_{k=l+1}^{k_T-1} 2^{(d+1)(k+1)}
        \mathcal{N} \left ( \mathcal{X}_{2^{-k}} \setminus \mathcal{X}_{2^{-(k+1)}} , \frac{2^{-(k+2)}}{L} \right )
        \notag \\
        &\qquad\qquad \times vol \left ( B_\infty\left(0,\frac{2^{-(k+3)}}{L}\right )\right ) \\
        & \geq \frac{v_1 \gamma }{2^{4d+d} L^d}  \sum_{k=l+1}^{k_T-1} 2^{(k+1)}
        \mathcal{N} \left ( \mathcal{X}_{2^{-k}} \setminus \mathcal{X}_{2^{-(k+1)}} , \frac{2^{-(k+1)}}{L} \right )
    \end{align}

    The missing term $k=l$ is nonnegative and can be absorbed into the constant $C$, which yields the stated sum from $k=l$. Here $v_1 := vol\left(B_\infty(0,1)\right )$, and we used monotonicity of packing numbers with respect to the radius. This concludes the proof.
\end{proof}

\subsection{Proof of Theorem 1}

With all the previous results, we are now equipped to prove \autoref{theorem:instance dependant integral regret bounds}.

Let us first restate the results :
\theoremRegretInstanceIntegral*

For this proof, we require \autoref{assump: nice opti sets} with $l=1$. 

\begin{proof}
Fix $\delta=T^{-3}$ and take $\delta_k=6\delta/(\pi^2k^2)$ so that $\sum_k\delta_k\le\delta$.
Condition on the global good event from Appendix~\ref{app:proof-first-lemma-complete}.

By Lemma~\ref{lemma:first regret bounds}, there exists a constant $C>0$ such that
\begin{equation}
\label{eq:thm1-start}
\begin{split}
R_T\ \le\ C\log(T)\Bigg[
&\sum_{k=1}^{k_T}2^k\,\mathcal{N}\bigl(\mathcal{X}_{2^{-k}}\setminus\mathcal{X}_{2^{-(k+1)}},2^{-k}/L\bigr)
\\
&\quad + 2^{-(k_T+1)}\sum_{k=1}^{k_T}2^{2k}\,\mathcal{N}\bigl(\mathcal{X}_{2^{-(k_T+1)}}\setminus\mathcal{X}^\star,2^{-k}/L\bigr)
\Bigg].
\end{split}
\end{equation}

While the two terms are complementary, we work on them separately for simplicity.

\paragraph{First term (up to phase $k_T-1$).}
Assumption~\ref{assump: nice opti sets} allows us to apply Lemma~\ref{lemma : sum integral eq}, yielding
\begin{equation}
\label{eq:thm1-first-term}
\sum_{k=1}^{k_T-1}2^k\,\mathcal{N}\bigl(\mathcal{X}_{2^{-k}}\setminus\mathcal{X}_{2^{-(k+1)}},2^{-k}/L\bigr)
\ \le\ C\int_{\mathcal{X}\setminus\mathcal{X}_{2^{-k_T}}}\frac{1}{\Delta(x)^{d+1}}\,dx.
\end{equation}

\paragraph{Terminal annulus and near-optimal term.}
The terminal annulus term corresponding to $k=k_T$ in the first sum of \eqref{eq:thm1-start} is:
\[
2^{k_T}\,\mathcal{N}\bigl(\mathcal{X}_{2^{-k_T}}\setminus\mathcal{X}_{2^{-(k_T+1)}}, 2^{-k_T}/L\bigr).
\]
Since $\mathcal{X}_{2^{-k_T}}\setminus\mathcal{X}_{2^{-(k_T+1)}} \subseteq \mathcal{X}_{2^{-k_T}}\setminus\mathcal{X}^\star$, the monotonicity of the packing number implies:
\[
2^{k_T}\,\mathcal{N}\bigl(\mathcal{X}_{2^{-k_T}}\setminus\mathcal{X}_{2^{-(k_T+1)}}, 2^{-k_T}/L\bigr)
\ \le\
2^{k_T}\,\mathcal{N}\bigl(\mathcal{X}_{2^{-k_T}}\setminus\mathcal{X}^\star, 2^{-k_T}/L\bigr).
\]
For the second term in \eqref{eq:thm1-start} (the near-optimal term), we can use the monotonicity of packing numbers to upper bound it as follows:
\begin{align}
    2^{-(k_T+1)}\sum_{k=1}^{k_T}2^{2k}\,\mathcal{N}\bigl(\mathcal{X}_{2^{-(k_T+1)}}\setminus\mathcal{X}^\star,2^{-k}/L\bigr) & \leq 2^{-k_T}\sum_{k=1}^{k_T}2^{2k}\,\mathcal{N}\bigl(\mathcal{X}_{2^{-k_T}}\setminus\mathcal{X}^\star,2^{-k}/L\bigr) \\
    & \leq \frac{4}{3}\,2^{-k_T} 2^{2k_T}\mathcal{N}\bigl(\mathcal{X}_{2^{-k_T}}\setminus\mathcal{X}^\star,2^{-k_T}/L\bigr) \\
    & = \frac{4}{3}\,2^{k_T} \mathcal{N}\bigl(\mathcal{X}_{2^{-k_T}}\setminus\mathcal{X}^\star,2^{-k_T}/L\bigr). \label{eq true for reg second sum}
\end{align}
Therefore, both the terminal annulus term $k=k_T$ and the near-optimal term are upper bounded (up to constant factors) by $2^{k_T} \mathcal{N}\bigl(\mathcal{X}_{2^{-k_T}}\setminus\mathcal{X}^\star,2^{-k_T}/L\bigr)$.

As in the proof of \autoref{lemma : sum integral eq}, we use the volumetric condition in \autoref{assump: nice opti sets} to relate this packing number to the volume integral over the near-optimal region. Indeed, let $(x_i)_{i \in \mathcal{N}\bigl(\mathcal{X}_{2^{-k_T}}\setminus\mathcal{X}^\star, 2^{-k_T}/L\bigr)}$ be a maximal packing of $\mathcal{X}_{2^{-k_T}}\setminus\mathcal{X}^\star$ at scale $2^{-k_T}/L$. The disjoint balls $B_\infty(x_i, 2^{-(k_T+3)}/L)$ are all contained within $\mathcal{X}_{2^{-(k_T-1)}} \setminus \mathcal{X}^\star$. Summing their volumes and applying \autoref{assump: nice opti sets} yields:
 \begin{equation} \label{eq: link integ and sum rest}
     \int_{\mathcal{X}_{2^{-k_T}} \setminus \mathcal{X}^\star} \frac{dx}{\left (2^{-k_T}\right )^{d+1}} \geq \frac{v_1 \gamma}{2^{4d+d} L^d} 2^{k_T}\mathcal{N}\bigl(\mathcal{X}_{2^{-k_T}}\setminus\mathcal{X}^\star,2^{-k_T}/L\bigr).
 \end{equation}

Combining \eqref{eq: link integ and sum rest}, \eqref{eq true for reg second sum}, and the terminal annulus bound, we can sum them with the first term's integral bound \eqref{eq:thm1-first-term} to obtain:
\begin{equation}
    R_T \leq \log(T)\left[ C_1 \int_{\mathcal{X}\setminus\mathcal{X}_{2^{-k_T}}}\frac{1}{\Delta(x)^{d+1}}\,dx + C_2 \int_{\mathcal{X}_{2^{-k_T}} \setminus \mathcal{X}^\star} \frac{dx}{\left (2^{-k_T}\right )^{d+1}} \right]
\end{equation}
where $C_2$ is a constant depending on $L, d, v_1, \gamma$.
This allows us to write overall:
\begin{equation*}
    R_T \leq C_3\log(T) \int_{\mathcal{X} \setminus \mathcal{X}^\star} \frac{dx}{\max\left(\Delta(x), 2^{-k_T} \right )^{d+1}} 
\end{equation*}

\end{proof}

We can also derive a refined version of \autoref{lemma : sum integral eq} which accounts for the local geometry of the set $\mathcal{X}$. The main advantage is that this allows us to get rid of the assumption that $\exists l, \mathcal{B} ( \mathcal{X}_{2^{-l}}, 2^{-l}) \subseteq \mathcal{X}$. Intuitively this is done by defining $d(x)$ and replacing $d$ with it in the integral formula. This is discussed in more detail in \autoref{app : Improved theorem 1}.

\subsection{Improved instance-dependent bound with tighter set-geometry dependence}\label{app : Improved theorem 1}

In this section, we both require lower assumptions on the geometry of the set $\mathcal{X}$ and provide tighter regret bounds, which account for the local geometry of the set and $f$ better. 

\begin{assumption}
    \label{assump : weak continuous set }
    For every $x \in \mathcal{X}$, the following integral is strictly positive $\int_{y \in B(x, \Delta(x)/4)} \cap \mathcal{X} dy$. 
\end{assumption}

\begin{lemma} [Series Integral inequality average dimension]
    \( \sum_{k=1}^{k_T} 2^{k} \mathcal{N} \left ( \mathcal{X}_{2^{-k}} \setminus \mathcal{X}_{2^{-(k+1)}}, 2^{-(k+1)}\right )  \leq \int_{\mathcal{X} \setminus \mathcal{X}_{k_T}} \frac{1}{\left ( f^\star - f(x) \right )^{d+1}} dx  \)
\end{lemma}

\begin{proof}
To ease the notations, we denote $N_k$ the set corresponding to the packing number $\mathcal{N} \left ( \mathcal{X}_{2^{-k}} \setminus \mathcal{X}_{2^{-(k+1)}}, 2^{-k}\right ) $ (if multiple such sets exist, we require $N_k$ be one of these sets).
We also use the following notation : $N_k^l := \{x \in N_k | 2^{l} \leq \nicefrac{\text{vol} \left ( B \left ( x,2^{-k} \right ) \cap \mathcal{X} \right )}{v_1 2^{-dk}} < 2^{l+1} \}$.

\begin{align}
    \sum_{k=1}^{k_T} 2^{k} \mathcal{N} \left ( \mathcal{X}_{2^{-k}} \setminus \mathcal{X}_{2^{-(k+1)}}, 2^{-k}\right )& \leq \sum_{k=1}^{k_T} 2^{k} \sum{x \in N_k} \frac{\text{vol}\left( B(x,2^{-k}) \cap \mathcal{X} \right)}{\text{vol}\left( B(x,2^{-k}) \cap \mathcal{X} \right)} \\
    & \leq \sum_{k=1}^{k_T} 2^{k} \sum_{l\geq 0} \sum_{x \in N_k^l} \frac{\text{vol}\left( B(x,2^{-k}) \cap \mathcal{X} \right)}{v_1 2^{-dk} 2^{l}} \\
    & \leq \sum_{k=1}^{k_T} 2^{k} \sum_{l\geq 0} \sum_{x \in N_k^l} \frac{\text{vol}\left( B(x,2^{-k}) \cap \mathcal{X} \right)}{v_1 \Delta(x)^d 2^{l}}\\
    & \leq \sum_{k=1}^{k_T} \sum_{l\geq 0}  \sum_{x \in N_k^l}  \frac{\text{vol}\left( B(x,\Delta(x)/2) \cap \mathcal{X} \right)}{v_1 \Delta(x)^{d+1} 2^{l}}\\
\end{align}

To match the above term we define $l_{\Delta}(x):= \sup \{ l  \in \mathbb{R}^+ | \frac{ vol\left( B(x,\Delta(x)/2) \cap \mathcal{X} \right)}{v_1 \Delta(x)^{d} 2^l} \geq 1 \} \cup \{ \infty \} $. 
Notice $l_{\Delta}(x) = \log_2 \left ( \frac{ vol\left( B(x,\Delta(x)/2) \cap \mathcal{X} \right)}{v_1 \Delta(x)^{d}} \right ).$

Given this, we can define directly  $d_{\Delta}(x):= \inf \{ d>0 | \frac{ vol\left( B(x,\Delta(x)/2) \cap \mathcal{X} \right)}{v_1 \Delta(x)^{d}} \leq 1 \}$.
By \autoref{assump : weak continuous set }, we have $d_{\Delta}(x) = \frac{\log \left ( vol(B(x,\Delta(x)/2) \cap \mathcal{X}) \right )-\log(v_1)}{\log (\Delta(x))}$. 

With this definition, we can do the following computations, let $ y \in B(x,\Delta(x)/2) \cap  \mathcal{X}$. Notice that, since $f$ is Lipschitz, $\Delta(y) \leq \nicefrac{3\Delta(x)}{2} $. Therefore, we have that $B(y, \Delta(y)/12) \subseteq B(x,\Delta(x)/2)$. These combined with the definition of $d_{\Delta}(y)$, gives us, for all $y \in  B(x,\Delta(x)/2) \cap  \mathcal{X}$: 
\begin{equation}
\frac{vol (B(x,\Delta(x)/2) \cap \mathcal{X})}{vol(B(x,\Delta(x)/2) \cap \mathcal{X})} \leq \frac{vol(B(x,\Delta(x)/2) \cap \mathcal{X})}{vol(B(y,\Delta(y)/12) \cap \mathcal{X})} \end{equation}

By using the definition of $d_{\Delta}(y)$, we have 
\begin{align}
    \frac{vol (B(x,\Delta(x)/2) \cap \mathcal{X})}{vol(B(x,\Delta(x)/2) \cap \mathcal{X})} & \leq \frac{vol (B(x,\Delta(x)/2) \cap \mathcal{X})}{v_1 \nicefrac{\Delta}{3}(y)^{d_{\nicefrac{\Delta}{3}}(y)}} \\
    & \leq \frac{3}{v_1} \int_{B(x,\Delta(x)/2) \cap \mathcal{X}} \frac{dy}{\Delta(y)^{d_{\nicefrac{\Delta}{3}}(y)}}
\end{align}

And obtain by summing  \begin{align}
    \sum_{k=1}^{k_T} 2^{k} \mathcal{N} \left ( \mathcal{X}_{2^{-k}} \setminus \mathcal{X}_{2^{-(k+1)}}, 2^{-k}\right ) & \leq \sum_{k=1}^{k_T} \int_{\mathcal{X}_{2^{-k}} \setminus \mathcal{X}_{2^{-(k+1)}}} \frac{dy}{\Delta(y)^{d_{\Delta(y)/8}+1}} \\
    & \leq \frac{8}{v_1} \int_{\mathcal{X} \setminus \mathcal{X}_{2^{-(k_T+1)}}} \frac{dy}{\Delta(y)^{d_{\Delta(y)/8}+1}}
\end{align}

Where we use the fact that points from the packing define disjoint ball of radius $2^{-k}$.

Outside of this proof, we \textbf{redefine} $d_{\Delta}(x)$ to be $d_{\nicefrac{\Delta}{8}}(x)$ for convenience.
\end{proof}

\subsection{Accompanying results in \autoref{sec:algo and inst bounds}}\label{appendix : further bandits results}

We provide below the complete proof of \autoref{lemma : lower bound}, which accompanies the main result in \autoref{sec:algo and inst bounds}.

We begin with the proof of \autoref{lemma : lower bound}. 

\begin{proof}
We leverage existing results from \cite{shekhar2022instance}, restated in \autoref{app : restated from litt}, which yield the following lower bound on the regret incurred by any algorithm.

We assume an $a_0$-consistent algorithm. We take
\[
a_0=\frac{d+1}{d+2},\qquad \Delta_T:=T^{-(1-a_0)}=T^{-\frac{1}{d+2}}.
\]
This choice of $a_0$ is the correct one; indeed the family of algorithms for which we want to control the minimum amount of regret on each specific instance $(\mathcal{X},f)$ are algorithms that already exhibit the worst-case regret bounds.
We know from the literature \cite{slivkins2019introduction} that the worst-case regret bound in the Lipschitz bandit setting is $\tilde{\Theta} \left ( T^{\nicefrac{d+1}{d+2}} \right )$.
We set $k_T := \lfloor \log_2 (T) / (d+2) \rfloor  $ such that, up to multiplicative constants, $2^{(d+2)k_T}\asymp T$.

We apply \autoref{prop:lipschitz-regret} with $\Delta=\Delta_T$ and define
\[
\mathcal{Z}_j:=\{x\in\mathcal{X}:2^j\Delta_T\le \Delta(x)<2^{j+1}\Delta_T\},\qquad
m_j:=\mathcal{N}(\mathcal{Z}_j,c\,2^j\Delta_T/L).
\]
The lower bound gives
\begin{equation}
\label{eq:lb-deltaT}
\Expectation[R_T]\ \ge\ c\sum_{j\ge 0}\frac{m_j}{2^{j+2}\Delta_T}.
\end{equation}

We reindex by $k=k_T-j-1$. Since $k_T=\lfloor \log_2(T)/(d+2) \rfloor$ and $\Delta_T=T^{-1/(d+2)}$, we have $2^{-(k_T+1)}<\Delta_T\le 2^{-k_T}$, hence
\[
2^{-(k+2)}<2^j\Delta_T\le 2^{-(k+1)},\qquad
2^{-(k+1)}<2^{j+1}\Delta_T\le 2^{-k}.
\]
Hence $\mathcal{X}_{2^{-k}}\setminus\mathcal{X}_{2^{-(k+1)}} \subseteq \mathcal{Z}_{j-1} \cup \mathcal{Z}_j \cup \mathcal{Z}_{j+1}$ and by subadditivity and monotonicity of packing numbers,
\[
m_{j-1} + m_j + m_{j+1} \ \ge\ \mathcal{N}(\mathcal{X}_{2^{-k}}\setminus\mathcal{X}_{2^{-(k+1)}},c'\,2^{-k}/L).
\]
Moreover $2^{j+2}\Delta_T\le 2^{-(k-1)}$, so $1/(2^{j+2}\Delta_T)\ge 2^{k-1}$. Summing over adjacent indices and adjusting the constant factor yields
\begin{equation}
\label{eq:lb-sum-outer}
\Expectation[R_T]\ \ge\ c\sum_{k=1}^{k_T}2^{k}\,\mathcal{N}(\mathcal{X}_{2^{-k}}\setminus\mathcal{X}_{2^{-(k+1)}},c'\,2^{-k}/L).
\end{equation}

For each $k$, let $A_k:=\mathcal{X}_{2^{-k}}\setminus\mathcal{X}_{2^{-(k+1)}}$ and $r_k:=c'2^{-k}/L$. Since a maximal $r_k$-packing covers $A_k$,
\[
\mathcal{N}(A_k,r_k)\ \ge\ \frac{\mathrm{vol}(A_k)}{v_1 r_k^d}\ \ge\ c\,2^{kd}\,\mathrm{vol}(A_k).
\]
On $A_k$ we have $\Delta(x)\ge 2^{-(k+1)}$, hence
\[
2^{k}\mathcal{N}(A_k,r_k)\ \ge\ c\,2^{(d+1)k}\,\mathrm{vol}(A_k)\ \ge\ c\,2^{-(d+1)}\int_{A_k}\frac{dx}{\Delta(x)^{d+1}}.
\]
Summing over $k\le k_T$ gives the first integral term
\[
R_T\ \ge\ c' \int_{\mathcal{X}\setminus\mathcal{X}_{2^{-k_T}}}\frac{dx}{\Delta(x)^{d+1}}.
\]

\end{proof}

\begin{proof}[Proof of Lemma \ref{lemma : relating zooming and optimal dimension}]
Let $d_z$ and $d^\star$ be the zooming dimension and optimal dimension of the instance $(\mathcal{X},f)$. By definition of the zooming dimension, for all $r>0$, we have $\mathcal{M}(\mathcal{X}_r \setminus \mathcal{X}^\star, r) \le \mathcal{O}(r^{-d_z})$. Let $y \in \mathcal{X}^\star$. Since $\mathcal{X}^\star$ is connected and there exists $r_0 > 0$ such that $\mathcal{B}(\mathcal{X}^\star, r_0) \subset \mathcal{X}$, the affine hull of $\mathcal{X}_r \setminus \mathcal{X}^\star$ contains $\mathcal{X}^\star$. 

For any small $r>0$, let $\mathbf{O}$ be an optimal $r$-covering of $\mathcal{X}_r \setminus \mathcal{X}^\star$. We construct a valid covering of $\mathcal{X}^\star$ at scale $r$ by adding points along the segments between each covering point in $\mathbf{O}$ and $y$. Since the maximum distance is bounded, this requires adding at most $\mathcal{O}(1/r)$ points per segment. This new set of points covers $\mathcal{X}^\star$ at scale $r$, so its size is bounded below by the minimal covering number $\mathcal{M}(\mathcal{X}^\star, r)$. Thus, we have:
\[ \mathcal{M}(\mathcal{X}^\star, r) \le \mathcal{O}(1/r) \cdot \mathcal{M}(\mathcal{X}_r \setminus \mathcal{X}^\star, r) \]
Substituting the dimension bounds gives $r^{-d^\star} \le \mathcal{O}(r^{-(d_z+1)})$. For this to hold as $r \to 0$, we must have $d^\star \le d_z + 1$.
\end{proof}

\section{Discussion}\label{sec:discussion-expert}

\subsection{The Lipschitz Expert Setting}

We clarify here the modeling choice underlying the full-information counterpart of the bandit setting introduced in \autoref{sec : introduction}. In the main paper, we described a setting that matches the uniformly Lipschitz expert model. We explain below the distinction between the Lipschitz expert and the uniformly Lipschitz expert models and how these two settings arise.

In the full-feedback analogue of the bandit setting in \autoref{sec : introduction}, it is more natural to reveal a random function $f_t:\mathcal X\to\mathbb R$ than a scalar perturbation $f(x)+\epsilon_t$, since an $x$-independent perturbation would preserve the maximizer of $f$ and therefore trivialize the problem. Following \cite{kleinberg2019bandits}, this leads to two related models.

In the \emph{Lipschitz experts} model, only the mean payoff
\(
f(x)=\mathbb E[f_t(x)]
\)
is assumed to be Lipschitz. In the \emph{uniformly Lipschitz experts} model, each realization $f_t$ is itself Lipschitz almost surely.

The distinction matters in the concentration step. In the Lipschitz experts model, concentration is obtained on a finite $\delta$-hitting set $S\subset\mathcal X$, and one must balance two competing effects: a smaller $\delta$ gives a better approximation of $\mathcal X$, but it increases $|S|$ and thus weakens the union-bound term in the deviation bound (cf \cite{kleinberg2019bandits} proof of Theorem~8.1 ).

In the uniformly Lipschitz experts model, this tradeoff is less severe. Since each $f_t$ is Lipschitz, one can control deviations of differences $f_t(x)-f_t(y)$, with bounds that scale with $D(x,y)$ rather than with a global union bound over all points. This allows a more local analysis and leads to sharper rates. 

\section{Expert setting background and proofs}\label{sec:expert-proofs}

In this section we focus on providing the necessary background and proofs of the Expert part of the paper. 

\subsection{Concentration bounds}

As in the bandit feedback analysis, the backbone of our algorithm and our proofs are concentration bounds which ensure that our empirical estimate concentrate towards their true average with high probability.

A notable improvement compared to the bandit feedback case is that we can use as many points as we want in our concentration as their observation is free (because of the full observation of each $f_t$). Yet, selecting too many such points might come at the cost of the tightness of our concentration bounds if we need to use union bounds on each point. 

This dilemma was notably addressed in \cite{kleinberg2019bandits} for generic metric spaces. In their setting, the log-covering number scales polynomially as $\log \mathcal{N}_\epsilon = \mathcal{O}(\epsilon^{-b})$, and a direct union bound over an $\epsilon$-net yields heavily suboptimal rates. They circumvent this by employing a covering tree and bounding the differences of empirical estimates using Chernoff bounds. 

However, in our setting where $\mathcal{X} \subset [0, 1]^d$, the metric geometry is much more benign. The log-covering number scales only logarithmically with the radius, i.e., $\log \mathcal{N}_\epsilon \le d \log(1/\epsilon)$. As we show below, this allows us to bypass the complex covering tree construction. By exploiting the $L$-Lipschitz property of both the empirical estimates and the mean reward function, we can directly extend the concentration from a carefully chosen $\epsilon$-net to the whole space $\mathcal{X}$ while preserving the optimal $\mathcal{O}(1/\sqrt{t})$ rate up to logarithmic factors.

\begin{lemma}[Uniform Concentration]\label{lem:uniform_concentration}
Let $f(x) = \mathbb{E}[f_t(x)]$ and $\hat{f}_t(x) = \frac{1}{t}\sum_{k=1}^t f_k(x)$. Under Assumption \ref{assump:Lipschitz}, for any $\delta \in (0, 1)$, with probability at least $1-\delta$, we have uniformly over all $t \ge 1$ and all $x \in \mathcal{X}$:
\[
\sup_{x \in \mathcal{X}} |\hat{f}_t(x) - f(x)| \le \frac{1}{4} \epsilon_t,
\]
where $\epsilon_t = 2L \sqrt{\frac{2d \log(t L) + \log(1/\delta_t)}{t}}$ with $\delta_t = \frac{6\delta}{\pi^2 t^2}$.
\end{lemma}

\begin{proof}
For any fixed time step $t \ge 1$, let $\delta_t = \frac{6 \delta}{\pi^2 t^2}$, so that $\sum_{t=1}^{T} \delta_t \leq \delta$. By the union bound over all $t \ge 1$, if we establish that for each $t$, uniform concentration holds with probability at least $1 - \delta_t$, then concentration holds uniformly over all $t \ge 1$ with probability at least $1 - \delta$.

Fix $t \ge 1$. Let $\epsilon > 0$ be a resolution scale to be chosen later, and let $S_\epsilon$ be an $\epsilon$-net of $\mathcal{X}$ with respect to the $\ell_\infty$ norm. Since $\mathcal{X} \subset [0, 1]^d$, we can construct such a net with cardinality $|S_\epsilon| \le \lceil 1/\epsilon \rceil^d$.

For any fixed $y \in S_\epsilon$, the random variables $f_k(y)$ are i.i.d., bounded in $[0, 1]$, and have expectation $f(y)$. By Hoeffding's inequality, for any $u > 0$:
\[
\mathbb{P}\left( |\hat{f}_t(y) - f(y)| > u \right) \le 2 \exp(-2 t u^2).
\]
Applying a union bound over all $y \in S_\epsilon$, we have with probability at least $1-\delta_t$:
\begin{equation}\label{eq:net_concentration}
\max_{y \in S_\epsilon} |\hat{f}_t(y) - f(y)| \le \sqrt{\frac{\log(2|S_\epsilon|/\delta_t)}{2t}} \le \sqrt{\frac{d \log(1/\epsilon + 1) + \log(2/\delta_t)}{2t}}.
\end{equation}

We now extend this concentration to the entire continuous space $\mathcal{X}$. For any arbitrary $x \in \mathcal{X}$, let $y(x) \in S_\epsilon$ be its closest point in the net, such that $\|x - y(x)\|_\infty \le \epsilon$. 
Since each $f_k$ is $L$-Lipschitz, the empirical average $\hat{f}_t$ is also $L$-Lipschitz. Similarly, the expected function $f$ is $L$-Lipschitz. We can thus bound the error at $x$ using the triangle inequality:
\begin{align*}
|\hat{f}_t(x) - f(x)| &\le |\hat{f}_t(x) - \hat{f}_t(y(x))| + |\hat{f}_t(y(x)) - f(y(x))| + |f(y(x)) - f(x)| \\
&\le L \|x - y(x)\|_\infty + |\hat{f}_t(y(x)) - f(y(x))| + L \|x - y(x)\|_\infty \\
&\le 2L\epsilon + \max_{y \in S_\epsilon} |\hat{f}_t(y) - f(y)|.
\end{align*}

Conditioning on the high-probability event of Equation \eqref{eq:net_concentration}, we obtain:
\[
\sup_{x \in \mathcal{X}} |\hat{f}_t(x) - f(x)| \le 2L\epsilon + \sqrt{\frac{d \log(1/\epsilon + 1) + \log(2/\delta_t)}{2t}}.
\]
To minimize this upper bound, we match the uncertainty level with the radius of the net by choosing $\epsilon = \frac{1}{L\sqrt{t}}$. Substituting this choice yields:
\[
\sup_{x \in \mathcal{X}} |\hat{f}_t(x) - f(x)| \le \frac{2}{\sqrt{t}} + \sqrt{\frac{d \log(L\sqrt{t} + 1) + \log(2/\delta_t)}{2t}} \le \frac{1}{4} \epsilon_t,
\]
where we set $\epsilon_t = 2L \sqrt{\frac{2d \log(t L) + \log(1/\delta_t)}{t}}$. This concludes the proof.
\end{proof}

\subsection{Proof of Theorem \ref{thm:full-info-instance}}
\label{app:proof-thm-expert}

We begin by restating the theorem. 

\InstanceDependentUniformlyExperts*

\begin{proof}
Let $\mathcal{E}$ be the uniform concentration event from Lemma \ref{lem:uniform_concentration}, which holds with probability at least $1-\delta$. On the complement event $\mathcal{E}^c$, we bound the regret crudely by $T \max \Delta \le T$. Since $\delta = T^{-3}$, this failure term contributes at most $1/T^2 \le o(1)$, which is negligible and can be accounted for by a constant term in the regret. We now conduct the proof conditioning on $\mathcal{E}$.

Recall that at each round $t \ge 2$, the algorithm chooses $x_t$ uniformly from the active set $\mathcal{A}_t$ defined using the information available up to round $t-1$:
\[
\mathcal{A}_t := \{ x \in \mathcal{X} \mid \hat{f}_{t-1}(x) \ge \max \hat{f}_{t-1} - \epsilon_{t-1} \}.
\]
At $t=1$, we set $\mathcal{A}_1 = \mathcal{X}$. The expected regret is bounded as:
\begin{align}\label{eq:regret_expert_base}
R_T &\le C_{ini} + \sum_{t=2}^T \mathbb{E}[f^\star - f(x_t) \mid \mathcal{F}_{t-1}] \notag \\
&\le C_{ini} + \sum_{t=2}^T \frac{1}{\text{vol}(\mathcal{A}_t)} \int_{\mathcal{A}_t \setminus \mathcal{X}^\star} \Delta(a) da.
\end{align}

Under the concentration event $\mathcal{E}$, the empirical means satisfy $\sup_{x} |\hat{f}_{t-1}(x) - f(x)| \le \frac{1}{4} \epsilon_{t-1}$ for all $t \ge 2$. This allows us to establish the two following properties:
For any $x \in \mathcal{X}_{\epsilon_{t-1}/2}$ (i.e., $\Delta(x) \le \epsilon_{t-1}/2$), we have:
    \[
    \hat{f}_{t-1}(x) \ge f(x) - \frac{1}{4}\epsilon_{t-1} \ge f^\star - \Delta(x) - \frac{1}{4}\epsilon_{t-1} \ge f^\star - \frac{3}{4}\epsilon_{t-1} \ge \max \hat{f}_{t-1} - \epsilon_{t-1}.
    \]
    To see why the last inequality holds, note that under the concentration event $\mathcal{E}$, we have $\hat{f}_{t-1}(y) \le f(y) + \frac{1}{4}\epsilon_{t-1} \le f^\star + \frac{1}{4}\epsilon_{t-1}$ for all $y \in \mathcal{X}$, which implies $\max \hat{f}_{t-1} \le f^\star + \frac{1}{4}\epsilon_{t-1}$ (or equivalently $f^\star \ge \max \hat{f}_{t-1} - \frac{1}{4}\epsilon_{t-1}$). This implies $\mathcal{X}_{\epsilon_{t-1}/2} \subseteq \mathcal{A}_t$, and thus $\text{vol}(\mathcal{A}_t) \ge \text{vol}(\mathcal{X}_{\epsilon_{t-1}/2})$.
Conversely, for any $x \in \mathcal{A}_t$, we have:
    \[
    f(x) \ge \hat{f}_{t-1}(x) - \frac{1}{4}\epsilon_{t-1} \ge \max \hat{f}_{t-1} - \frac{5}{4}\epsilon_{t-1} \ge \hat{f}_{t-1}(x^\star) - \frac{5}{4}\epsilon_{t-1} \ge f^\star - \frac{3}{2}\epsilon_{t-1},
    \]
    which implies $\Delta(x) \le 2\epsilon_{t-1}$. Thus, $\mathcal{A}_t \subseteq \mathcal{X}_{2\epsilon_{t-1}}$.

Substituting back into \eqref{eq:regret_expert_base}, we obtain:
\begin{equation}\label{eq:expert_deterministic_sum}
R_T \;\le\; 1 + \sum_{t=2}^T \frac{1}{\text{vol}(\mathcal{X}_{\epsilon_{t-1}/2})} \int_{\mathcal{X}_{2\epsilon_{t-1}} \setminus \mathcal{X}^\star} \Delta(a) da.
\end{equation}

To evaluate this sum, we partition the domain $\mathcal{X}_{2\epsilon_1} \setminus \mathcal{X}^\star$ into time-based regions $\mathcal{X}_{2\epsilon_{t'-1}} \setminus \mathcal{X}_{2\epsilon_{t'}}$ for $t' \ge 2$, and a residual region $\mathcal{X}_{2\epsilon_T} \setminus \mathcal{X}^\star$.

The sum of integrals in \eqref{eq:expert_deterministic_sum} can be rewritten by exchanging the order of summation:
\begin{align*}
\sum_{t=2}^T \frac{1}{\text{vol}(\mathcal{X}_{\epsilon_{t-1}/2})} \int_{\mathcal{X}_{2\epsilon_{t-1}} \setminus \mathcal{X}^\star} \Delta(a) da
&= \sum_{t=2}^T \frac{1}{\text{vol}(\mathcal{X}_{\epsilon_{t-1}/2})} \left( \sum_{t'=t}^T \int_{\mathcal{X}_{2\epsilon_{t'-1}} \setminus \mathcal{X}_{2\epsilon_{t'}}} \Delta(a) da + \int_{ \mathcal{X}_{2\epsilon_{T}} \setminus \mathcal{X}^\star} \Delta(a) da \right) \\
&= \sum_{t'=2}^{T} \int_{\mathcal{X}_{2\epsilon_{t'-1}} \setminus \mathcal{X}_{2\epsilon_{t'}}} \Delta(a) \left( \sum_{t=2}^{t'} \frac{1}{\text{vol}(\mathcal{X}_{\epsilon_{t-1}/2})} \right) da + \int_{ \mathcal{X}_{2\epsilon_{T}} \setminus \mathcal{X}^\star} \Delta(a) \left( \sum_{t=2}^T \frac{1}{\text{vol}(\mathcal{X}_{\epsilon_{t-1}/2})} \right) da.
\end{align*}

For any $a \in \mathcal{X}_{2\epsilon_{t'-1}} \setminus \mathcal{X}_{2\epsilon_{t'}}$, we know that $2\epsilon_{t'} < \Delta(a) \le 2\epsilon_{t'-1}$. Therefore, for any $t \le t'$, we have $\epsilon_{t-1} \ge \epsilon_{t'-1} \ge \Delta(a)/2$, which implies $\epsilon_{t-1}/2 \ge \Delta(a)/4$. By the monotonicity of the level sets, we have $\text{vol}(\mathcal{X}_{\epsilon_{t-1}/2}) \ge \text{vol}(\mathcal{X}_{\Delta(a)/4})$.

Moreover, from the definition of $\epsilon_{t'} = \tilde{\Theta}((t')^{-1/2})$, the condition $2\epsilon_{t'} < \Delta(a)$ implies that $t' \le C_0 \frac{\log T}{\Delta(a)^2}$ for some constant $C_0 > 0$. The inner sum over $t$ is therefore bounded by:
\[
\sum_{t=2}^{t'} \frac{1}{\text{vol}(\mathcal{X}_{\epsilon_{t-1}/2})} \;\le\; \frac{t'}{\text{vol}(\mathcal{X}_{\Delta(a)/4})} \;\le\; C_0 \frac{\log T}{\Delta(a)^2 \text{vol}(\mathcal{X}_{\Delta(a)/4})}.
\]

Multiplying by $\Delta(a)$ and integrating over each region $\mathcal{X}_{2\epsilon_{t'-1}} \setminus \mathcal{X}_{2\epsilon_{t'}}$, we obtain for the first part of the sum:
\[
\int_{\mathcal{X}_{2\epsilon_{t'-1}} \setminus \mathcal{X}_{2\epsilon_{t'}}} \Delta(a) \left( \sum_{t=2}^{t'} \frac{1}{\text{vol}(\mathcal{X}_{\epsilon_{t-1}/2})} \right) da \;\le\; C_0 \log T \int_{\mathcal{X}_{2\epsilon_{t'-1}} \setminus \mathcal{X}_{2\epsilon_{t'}}} \frac{1}{\Delta(a) \text{vol}(\mathcal{X}_{\Delta(a)/4})} da.
\]

Summing this bound over all $t'$ from $2$ to $T$ yields the integral bound over the resolved region:
\[
C_1 \log T \int_{\mathcal{X}_{2\epsilon_1} \setminus \mathcal{X}_{2\epsilon_T}} \frac{dx}{\Delta(x) \text{vol}(\mathcal{X}_{\Delta(x)/4})} \;\le\; C_1 \log T \int_{\mathcal{X} \setminus \mathcal{X}_{2\epsilon_T}} \frac{dx}{\Delta(x) \text{vol}(\mathcal{X}_{\Delta(x)/4})}.
\]
Note that the volumes $\text{vol}(\mathcal{X}_{\Delta(x)/4})$ and $\text{vol}(\mathcal{X}_{\Delta(x)/2})$ are equivalent up to constant factors under standard scaling assumptions on the level sets, and this constant is absorbed by $C_1$.

For the residual region $\mathcal{X}_{2\epsilon_T} \setminus \mathcal{X}^\star$, which corresponds to the second part of the sum, we can rewrite it back as:
\[
\int_{ \mathcal{X}_{2\epsilon_{T}} \setminus \mathcal{X}^\star} \Delta(a) \left( \sum_{t=2}^T \frac{1}{\text{vol}(\mathcal{X}_{\epsilon_{t-1}/2})} \right) da = \sum_{t=2}^T \frac{1}{\text{vol}(\mathcal{X}_{\epsilon_{t-1}/2})} \int_{\mathcal{X}_{2\epsilon_T} \setminus \mathcal{X}^\star} \Delta(a) da.
\]

Combining these two bounds yields the stated regret guarantee of Theorem~\ref{thm:full-info-instance}.
\end{proof}

\subsection{Proof of Corollary \ref{lem:full-info-zooming}}
\label{proof: corollary full info}

We now provide the formal proof for the zooming-dimension upper bound in the full-information setting.

\begin{proof}
Let $M = \max(d_z, d^\star)$. The instance-dependent regret bound of Theorem~\ref{thm:full-info-instance} can be decomposed into two main components, $R_T \le \mathcal{O}(\log T) (I_1 + I_2)$, where
\begin{align*}
    I_1 &= \int_{\mathcal{X} \setminus \mathcal{X}_{\epsilon_T}} \frac{1}{\Delta(x) \text{vol}(\mathcal{X}_{\Delta(x)/2})} dx, \\
    I_2 &= \sum_{t=1}^T \frac{1}{\text{vol}(\mathcal{X}_{\epsilon_t})} \int_{\mathcal{X}_{\epsilon_T}} \Delta(x) dx.
\end{align*}
Recall from the proof of Theorem~\ref{thm:full-info-instance} that $\epsilon_t = \tilde{\Theta}(t^{-1/2})$.

\paragraph{Geometric volume bounds.}
By the definition of the zooming dimension $d_z$, the maximum number of disjoint $\ell_\infty$-balls of radius $r$ that can be packed in the region $\mathcal{X}_{2r} \setminus \mathcal{X}_r$ is bounded by $c_z r^{-d_z}$ for some constant $c_z > 0$. Since the action space $\mathcal{X}$ is embedded in $\mathbb{R}^d$, the volume of each such $\ell_\infty$-ball of radius $r$ is $(2r)^d \propto r^d$. Because these balls are mutually disjoint, the volume of their union is simply the sum of their individual volumes. Consequently, we can upper bound the volume of the annulus by the number of balls multiplied by their volume:
\begin{equation}
    \text{vol}(\mathcal{X}_{2r} \setminus \mathcal{X}_r) \le c_z r^{-d_z} \cdot (2r)^d = C_z r^{d - d_z},
\end{equation}
for some constant $C_z > 0$. Applying the exact same reasoning to the near-optimal dimension $d^\star$, we obtain an analogous bound for the near-optimal region: $\text{vol}(\mathcal{X}_r) \le C_\star r^{d - d^\star}$. By summing the volumes of dyadic annuli geometrically, we also naturally have $\text{vol}(\mathcal{X}_r) \le C_z' r^{d - d_z}$.

We can similarly derive a volume lower bound. Because $d_z$ is the \emph{smallest} exponent characterizing the packing number of the annuli, and $d^\star$ is the smallest exponent characterizing the near-optimal region, the packing number of $\mathcal{X}_r$ requires at least $\Omega(r^{-M})$ balls of radius $r$, where $M = \max(d_z, d^\star)$. This directly implies a lower bound on the volume of $\mathcal{X}_r$:
\begin{equation}
    \text{vol}(\mathcal{X}_r) \ge c_1 r^{d - M},
\end{equation}
for some universal constant $c_1 > 0$ independent of $T$.

\paragraph{Bounding the first term (Peeling argument).}
We decompose the domain of integration $\mathcal{X} \setminus \mathcal{X}_{\epsilon_T}$ into disjoint dyadic annuli. Let $J = \lceil \log_2(1/\epsilon_T) \rceil$. We define the annuli as $A_j = \mathcal{X}_{2^{-j+1}} \setminus \mathcal{X}_{2^{-j}}$ for $j = 1, \dots, J$. 
For any $x \in A_j$, we have $\Delta(x) \in (2^{-j}, 2^{-j+1}]$. In particular, $\Delta(x) > 2^{-j}$, which allows us to lower bound the volume of the inner optimal set using our established lower bound:
\[
    \text{vol}(\mathcal{X}_{\Delta(x)/2}) \ge \text{vol}(\mathcal{X}_{2^{-(j+1)}}) \ge c_1 2^{-(j+1)(d-M)}.
\]
Substituting this, along with the volume upper bound $\text{vol}(A_j) \le C_z 2^{-j(d-d_z)}$, into the integral over $A_j$, we obtain:
\begin{align*}
    \int_{A_j} \frac{1}{\Delta(x) \text{vol}(\mathcal{X}_{\Delta(x)/2})} dx 
    &\le \frac{\text{vol}(A_j)}{2^{-j} \cdot c_1 2^{-(j+1)(d-M)}} \\
    &\le \frac{C_z 2^{-j(d-d_z)}}{c_1 2^{-M - j(1+d-M)}} \\
    &= C' 2^{j(1 + d_z - M)}.
\end{align*}
Summing this bound over all annuli $j=1, \dots, J$, we have:
\begin{equation}
    I_1 \le C' \sum_{j=1}^J 2^{j(1 + d_z - M)} \le C'' 2^{J \max(0, 1 + d_z - M)}.
\end{equation}
Since $2^J \approx \epsilon_T^{-1} = \tilde{\mathcal{O}}(T^{1/2})$, we have $I_1 = \tilde{\mathcal{O}}\left(T^{\frac{1 + d_z - M}{2}}\right)$. Note that by Lemma~\ref{lemma : relating zooming and optimal dimension}, we naturally have $M \le d_z + 1$, meaning $1 + d_z - M \ge 0$, and thus the exponent is non-negative.

\paragraph{Bounding the residual term.}
For the second term $I_2$, we upper bound the inner integral by the maximum suboptimality gap in $\mathcal{X}_{\epsilon_T}$ multiplied by its volume:
\begin{equation}
    \int_{\mathcal{X}_{\epsilon_T}} \Delta(x) dx \le \epsilon_T \text{vol}(\mathcal{X}_{\epsilon_T}) \le C_z' \epsilon_T^{1 + d - d_z}.
\end{equation}
For the sum over time steps, we use the volume lower bound $\text{vol}(\mathcal{X}_{\epsilon_t}) \ge c_1 \epsilon_t^{d-M}$. Plugging these bounds into $I_2$ yields:
\begin{align*}
    I_2 &\le C_z' \epsilon_T^{1 + d - d_z} \sum_{t=1}^T \frac{1}{c_1 \epsilon_t^{d-M}} \\
    &= \mathcal{O}\left( \epsilon_T^{1 + d - d_z} \right) \sum_{t=1}^T \mathcal{O}\left( \epsilon_t^{-(d-M)} \right) \\
    &= \tilde{\mathcal{O}}\left( T^{-\frac{1 + d - d_z}{2}} \right) \sum_{t=1}^T \tilde{\mathcal{O}}\left( t^{\frac{d-M}{2}} \right) \\
    &= \tilde{\mathcal{O}}\left( T^{-\frac{1 + d - d_z}{2}} \cdot T^{1 + \frac{d-M}{2}} \right) \\
    &= \tilde{\mathcal{O}}\left( T^{\frac{2 - 1 - d + d_z + d - M}{2}} \right) \\
    &= \tilde{\mathcal{O}}\left( T^{\frac{1 + d_z - M}{2}} \right).
\end{align*}

\paragraph{Conclusion.}
Combining the bounds for $I_1$ and $I_2$, and recalling that $M = \max(d_z, d^\star)$, we conclude that the total regret is bounded by:
\[
    R_T \le \tilde{\mathcal{O}}\left( T^{\frac{1 + d_z - \max(d_z, d^\star)}{2}} \right),
\]
which rigorously completes the proof.
\end{proof}

\section{Relaxing the Lipschitz assumption}\label{sec : relaxed}

As we mentioned in the introduction, there is a growing literature which studies the ability of agents to interact with markets (or conversely), and models this problem as a continuum-armed online learning with partial feedback problem. Many of these problems exhibit local regularities which are weaker than the global Lipschitz condition we analyze above. Motivated by dynamic pricing and auction settings \citep{bu2022context,branzei2023learning}, we discuss below how the analysis and guarantees obtained under the global Lipschitz condition (Assumption~\ref{assump:Lipschitz}) extend to weaker regularity assumptions. Throughout, we keep the same action space $\mathcal{X}\subset[0,1]^d$ and metric $d_\infty(x,y)=\|x-y\|_\infty$. As elsewhere in the paper, we assume that $f$ attains its maximum on $\mathcal{X}$ so that $\mathcal{X}^\star\neq\emptyset$.

\subsection{A relaxation requiring a change: one-sided Lipschitzness}\label{app:one-sided-lipschitz}

We now consider a different relaxation in which regularity is required only along a \emph{forward} direction. In one dimension, this corresponds to a right-Lipschitz condition: the increase $f(y)-f(x)$ is controlled when $y\ge x$, while no constraint is imposed when moving to the left. In $\mathbb{R}^d$, we express the forward direction via the coordinate-wise partial order: for $x,y\in\mathbb{R}^d$ we write $x\preceq y$ if $x_j\le y_j$ for all $j\in[d]$.

\begin{assumption}[One-sided Lipschitzness (forward pairs only)]
\label{assump : Lipschitz one sided}
There exists $L>0$ such that for all $x,y\in\mathcal{X}$ with $x\preceq y$,
\[
f(y)-f(x) \;\le\; L\|y-x\|_\infty.
\]
\end{assumption}

\begin{remark}
Assumption~\ref{assump : Lipschitz one sided} only constrains increases of $f$ along the forward orthant; it imposes no condition on pairs $(x,y)$ that are not comparable under $\preceq$. 
\end{remark}

We focus below on showing how both our algorithm PACO (\autoref{algorithm: high level}) and the analysis we conducted can be adapted to this relaxed setting.
As before, our analysis relies on the quality of our discretization oracle. Yet, since we consider functions satisfying only \autoref{assump : Lipschitz one sided}, some of the ability to generalize the information from pointwise-estimation of $f$ neighbouring points is lost.
In fact we are only able to use estimates for points in the forward direction of the estimated point. 

This observation motivates the following change in the requirement to our Discretization oracle: instead of requiring a net which covers the entire space, we require a net which covers the queried set along the forward direction. Formally, : 

\subsection{Discretization oracle}

Define the forward closed balls $B_\infty^+(x,r)$ as the set of points "\emph{forward}" of $x$ at distance at most $r$. Formally, $B_\infty^+(x,r):=\{ y \in \mathcal{X}, \|x-y\|_\infty \leq r \text{ and } y \succeq x \}  $. 
With this notation, we can now define the property required for our discretization oracle: it needs to return a "forward covering" net.

Formally, we assume that given a set $\mathcal{Y}\subseteq[0,1]^d$ and a radius $r\in(0,1]$, the forward discretization oracle returns a finite set of points $\mathbf{O}_f(\mathcal{Y},r)=\{y_1,\ldots,y_n\}\subseteq \mathcal{Y}$ such that $\mathcal{Y}$ is covered by the above defined forward balls of radius $r$:
\[
\mathcal{Y}\subseteq \bigcup_{i=1}^n B_\infty^+(y_i,r).
\]

\begin{remark}
We introduced the above notations and defined forward balls to make the parallel with the approach in \autoref{sec:algo and inst bounds} evident, but the forward balls are nothing more than $\ell_\infty$ closed balls of radius $r/2$ centered in $x+\frac{r}{2}\mathbf{1}$.  
\end{remark}

Using the above remarks, most of the analysis performed in \autoref{sec:algo and inst bounds} remains valid, provided that we replace the standard discretization oracle by the one defined above.

For the sake of completeness, we describe here the modified version of our algorithm which accommodates the one-sided Lipschitz assumption instead of the global one (\autoref{assump : Lipschitz one sided}).

\begin{algorithm}[htbp]
\caption{PACO (one-sided, forward cover)}\label{algorithm: high level one sided}
\begin{algorithmic}[1]
\State \textbf{Input:} time horizon $T$, action set $\mathcal{X}$, confidence level $\delta\in(0,1)$, Lipschitz constant $L$
\State \textbf{Initialize:} $t\gets 1$, $k\gets 1$, active region $\mathcal{A}_1 \gets \mathcal{X}$
\While{$t \le T$}
    \State $r_k \gets 2^{-k}$
    \State $\delta_k \gets \frac{6\delta}{\pi^2 k^2}$
    \State Discretize: $S_k \gets \mathbf{O}_f(\mathcal{A}_k, \nicefrac{r_k}{L})$
    \State Run Algorithm~\ref{algo : successive elimination one sided} on $S_k$ with parameters $(r_k,\delta_k, t, T)$, obtain survivors $\widehat{S}_k$ and updated $t$
    \State Update active region: $\mathcal{A}_{k+1} \gets \mathcal{A}_k \cap \bigcup_{a\in\widehat{S}_k} B_\infty^+(a,\nicefrac{r_k}{L}) $
    \State \Comment{$t$ is updated inside Algorithm~\ref{algo : successive elimination one sided}}
    \State $k\gets k+1$
\EndWhile
\end{algorithmic}
\end{algorithm}

\begin{algorithm}[htbp]
\caption{Successive elimination (one-sided setting)}\label{algo : successive elimination one sided}
\begin{algorithmic}[1]
\State \textbf{Input:} finite set of arms $S_k$, accuracy $r_k\in(0,1]$, confidence budget $\delta_k\in(0,1)$, time $t$, horizon $T$
\State \textbf{Output:} surviving set $\widehat{S}_k$, updated $t$
\State \textbf{Initialize:} $\ell\gets 1$, $A_\ell \gets S_k$
\While{$u_{k,\ell} > r_k/4$ \textbf{and} $t \le T$}
    \State Pull each $a\in A_\ell$ once (abort if $t > T$); update empirical means $\widehat{\mu}_\ell(a)$ and $t \gets t + 1$
    \State \Comment{For any $a\in A_\ell$, it has been pulled exactly $\ell$ times so far.}
    \State Set confidence radius:
    \State $u_{k,\ell} \gets \sqrt{\frac{2}{\ell}\log\!\Big(\frac{\pi^2\,\ell^2\,|S_k|}{6\,\delta_k}\Big)} $
    \State Update the active set:
    \State $A_{\ell+1} \gets \left\{a\in A_\ell:\ \widehat{\mu}_\ell(a)+u_{k,\ell} \ge \max_{j\in A_\ell}\widehat{\mu}_\ell(j)-u_{k,\ell} -r_k\right\}$
    \State $\ell\gets \ell+1$
\EndWhile
\State \Return $\widehat{S}_k \gets A_\ell$
\end{algorithmic}
\end{algorithm}

\paragraph{Modified active-set update.}
In the one-sided setting, the successive-elimination subroutine remains unchanged, but the discretization and active-set update must be made directional. At phase $k$, with $r_k=2^{-k}$, we take a finite set
\[
S_k\subseteq\mathcal{A}_k
\qquad\text{such that}\qquad
\mathcal{A}_k\subseteq\bigcup_{a\in S_k} (B_\infty^+(a,r_k/L) \cap \mathcal{A}_k),
\]
and after elimination we define
\[
\mathcal{A}_{k+1}\ :=\ \mathcal{A}_k \cap \bigcup_{a\in\widehat S_k} B_\infty^+(a,r_k/L).
\]
This is the natural analogue of PACO for \autoref{assump : Lipschitz one sided}: if $x\in B_\infty^+(a,r_k/L) \cap \mathcal{A}_k$, then $a\preceq x$ and $\|x-a\|_\infty\le r_k/L$, hence
\[
f(x)\ \le\ f(a)+r_k.
\]
Therefore, once the value of $a$ is known up to accuracy $r_k$, the whole forward box anchored at $a$ can be treated conservatively.

\begin{proposition}[Optimal points remain active in the one-sided setting]
\label{prop:onesided-optimal-stays-active}
Consider the directional variant of PACO described above and condition on the global good event $\mathcal{E}$ from Definition~\ref{def:global-good}. Then for every completed phase $k\le k_T$ and every $x^\star\in\mathcal{X}^\star$, there exists a point $a\in \widehat S_k$ such that
\[
x^\star\in B_\infty^+(a,r_k/L) \cap \mathcal{A}_k
\qquad\text{and}\qquad
\Delta(a)\le r_k.
\]
In particular, $\mathcal{X}^\star\subseteq\mathcal{A}_{k+1}$ for every completed phase $k$.
\end{proposition}

\begin{proof}
We argue exactly as in Lemma~\ref{lem:opt-stays-active-complete}, replacing ordinary coverings by forward coverings. Let $x^\star\in\mathcal{X}^\star\cap\mathcal{A}_k$. By construction of $S_k$, there exists $a\in S_k$ such that $x^\star\in B_\infty^+(a,r_k/L) \cap \mathcal{A}_k$. Hence $a\preceq x^\star$ and $\|x^\star-a\|_\infty\le r_k/L$. \autoref{assump : Lipschitz one sided} then gives
\[
f^\star-f(a)\ =\ f(x^\star)-f(a)\ \le\ L\|x^\star-a\|_\infty\ \le\ r_k,
\]
so $\Delta(a)\le r_k$. By Lemma~\ref{lem:nearopt-never-elim}, the point $a$ is never eliminated and therefore belongs to $\widehat S_k$. Since $x^\star\in B_\infty^+(a,r_k/L) \cap \mathcal{A}_k$, it also belongs to $\mathcal{A}_{k+1}$. This proves the claim.
\end{proof}

\begin{proposition}[Per-arm pull bound remains valid]
\label{prop:onesided-pull-bound}
Under the same assumptions, for every completed phase $k\le k_T$ and every $a\in S_k$,
\[
N_k(a)
\ \le\
c_0\,\frac{\log\!\big(\frac{|S_k|}{\delta_k}\big)}{\max\{\Delta(a),r_k\}^2},
\]
where $c_0>0$ is the constant from Lemma~\ref{lem:pull-bound-complete}.
\end{proposition}

\begin{proof}
The proof of Lemma~\ref{lem:pull-bound-complete} only uses the existence, in each phase, of at least one benchmark point $b\in S_k$ satisfying $\Delta(b)\le r_k$ and surviving until the end of the phase. Proposition~\ref{prop:onesided-optimal-stays-active} provides exactly such a point in the one-sided setting. Once this benchmark point is available, the elimination argument for any arm $a\in S_k$ is unchanged. The bound therefore follows verbatim.
\end{proof}

\paragraph{What no longer holds}

While the above described results continue to hold true under \autoref{assump : Lipschitz one sided}, the same is not true for \autoref{lem:gap-next-active-complete}. Indeed as there are no longer any constraints on how quickly the function can decrease, any new points in $S_{k+1}$ which did not belong to $S_k$ may be as suboptimal as one wants. 

One can nevertheless recover the results of \autoref{lemma:first regret bounds}. Indeed, \autoref{lem:gap-next-active-complete} is only used to ensure summing from $k-3$ is sufficient in \eqref{eq: temporary instance delta 2}. Yet, noticing that for all index $k'\leq k$ in \eqref{eq:temporary instance delta 3}, we can replace the upper bound $\frac{\left |S_k^{k'}\right |}{r_{k}}$ by $\frac{\left |S_k^{k'}\right |}{r_{k'}}$ allows to safeguard our analysis. 
Indeed, in the (slightly modified version of \eqref{eq:temporary instance delta 3}), each term $\frac{\mathcal{N}(\mathcal{X}_{r_{k'}}\setminus \mathcal{X}_{r_{k'+1}},r_k)}{r_{k'}}$ is dominated by $\frac{\mathcal{N}(\mathcal{X}_{r_{k'}}\setminus \mathcal{X}_{r_{k'+1}},r_k)}{r_{k}}$ by monotonicity of the covering number w.r.t the radius. 

The rest of the analysis is then identical (up to changed universal constants).

\paragraph{What must be changed in the proof of Theorem~\ref{theorem:instance dependant integral regret bounds}.}
The proof of Theorem~\ref{theorem:instance dependant integral regret bounds} does not carry over verbatim, but the obstruction is localized. The only genuinely symmetric step in Appendix~\ref{appendix : further bandits results} is the series-to-integral argument of Lemma~\ref{lemma : sum integral eq}, where each packed point in a gap annulus is thickened by an $\ell_\infty$ ball and global Lipschitzness is used to keep this whole ball inside a slightly enlarged near-optimal region. Under \autoref{assump : Lipschitz one sided}, the correct replacement is to work with \emph{backward boxes}. For $x\in\mathbb{R}^d$ and $r>0$, define
\[
B_\infty^{-}(x,r)\ :=\ \{y\in\mathbb{R}^d:\ \|x-y\|_\infty \le r \text{ and } y \preceq x\}.
\]
If $y\in B_\infty^{-}(x,r)$, then $y\preceq x$ and $\|x-y\|_\infty\le r$, so \autoref{assump : Lipschitz one sided} yields
\[
f(x)-f(y)\ \le\ Lr,
\qquad\text{hence}\qquad
\Delta(y)\ \le\ \Delta(x)+Lr.
\]
Thus one-sided Lipschitzness still propagates \emph{small gap values}, but only in the backward direction.

To recover the integral proof, the symmetric interior condition of Assumption~\ref{assump: nice opti sets} should therefore be replaced by a directional variant, for instance:
\begin{assumption}[Backward thickness near near-optimal points]
\label{assump:onesided-thickness}
There exist $l\in\mathbb{N}$ and $\gamma\in(0,1]$ such that for all $k\ge l$ and all $x\in\mathcal{X}_{2^{-k}}\setminus\mathcal{X}^\star$,
\[
\mathrm{vol}\Bigl(B_\infty^{-}\bigl(x,2^{-(k+3)}/L\bigr) \cap (\mathcal{X}\setminus\mathcal{X}^\star)\Bigr)
\ \ge\
\gamma\Bigl(2^{-(k+3)}/L\Bigr)^d.
\]
\end{assumption}

Under Assumption~\ref{assump:onesided-thickness}, the proof of Lemma~\ref{lemma : sum integral eq} can be repeated with only one substantive modification. Instead of associating a symmetric ball to each point of a packing of
\[
\mathcal{X}_{2^{-k}}\setminus\mathcal{X}_{2^{-(k+1)}},
\]
one associates a backward box of sidelength of order $2^{-k}/L$. For $x_i$ in such an annulus and $y\in B_\infty^{-}(x_i,2^{-(k+3)}/L)$, one has
\[
\Delta(y)\ \le\ \Delta(x_i)+2^{-(k+3)}\ \le\ C\,2^{-k}
\]
for a universal constant $C$. Hence the truncated integrand
\[
x\mapsto \frac{1}{\max\{\Delta(x),2^{-k_T}\}^{d+1}}
\]
is bounded below by a constant multiple of $2^{(d+1)k}$ on each such backward box whenever $k\le k_T$. Summing over a packing therefore yields the one-sided analogue of the series-to-integral bound,
\[
\int_{\mathcal{X}\setminus\mathcal{X}^\star}
\frac{dx}{\max\{\Delta(x),2^{-k_T}\}^{d+1}}
\ \gtrsim\
\sum_{k=1}^{k_T} 2^k\,
\mathcal{N}\Bigl(\mathcal{X}_{2^{-k}}\setminus\mathcal{X}_{2^{-(k+1)}},c\,2^{-k}/L\Bigr),
\]
for a universal constant $c>0$.

The near-optimal term is handled in the same way: one applies the same backward-box argument to points of $\mathcal{X}_{2^{-k_T}}\setminus\mathcal{X}^\star$, which suffices because the theorem involves the truncated integrand rather than the singular integrand $\Delta(x)^{-(d+1)}$. Consequently, once the packing-type regret bound of Lemma~\ref{lemma:first regret bounds} has been established in the one-sided setting, the proof of Theorem~\ref{theorem:instance dependant integral regret bounds} goes through with the same truncated-integral form, up to different universal constants and with Assumption~\ref{assump: nice opti sets} replaced by Assumption~\ref{assump:onesided-thickness}.

\subsection{A logarithmic guarantee in the positive-gap regime}

The one-sided setting nonetheless admits a simple and useful positive result when the instance has a strictly positive global suboptimality gap.

\begin{definition}[Global suboptimality gap]
\label{def:global-gap-onesided}
Define
\[
\Delta_f\ :=\ \inf_{x\in\mathcal{X}\setminus\mathcal{X}^\star} \Delta(x)\in[0,\infty],
\]
with the convention $\Delta_f=\infty$ when $\mathcal{X}=\mathcal{X}^\star$.
\end{definition}

When $\Delta_f>0$, a single forward discretization at the right scale already contains an optimal representative. The problem therefore reduces to a finite-armed stochastic bandit.

\begin{theorem}[Positive-gap reduction to a finite-armed problem]
\label{thm:onesided-positive-gap}
Assume \autoref{assump : Lipschitz one sided} and suppose that $\Delta_f>0$. Let
\[
r_\Delta\ :=\ \frac{\Delta_f}{4L},
\]
and let $S\subseteq\mathcal{X}$ be any forward cover of $\mathcal{X}$ at scale $r_\Delta$, with
\[
|S|\ =\ \mathcal{M}_{+}(\mathcal{X},r_\Delta).
\]
Then $S\cap\mathcal{X}^\star\neq\emptyset$. Consequently, if one runs any standard stochastic finite-armed bandit algorithm on the finite action set $S$, there exists a universal constant $C>0$ such that
\[
\Expectation[R_T]
\ \le\
C\sum_{a\in S\setminus\mathcal{X}^\star}\frac{\log T}{\Delta(a)}
\ \le\
C\,\mathcal{M}_{+}\!\left(\mathcal{X},\frac{\Delta_f}{4L}\right)\frac{\log T}{\Delta_f}.
\]
\end{theorem}

\begin{proof}
Let $x^\star\in\mathcal{X}^\star$. Since $S$ is a forward cover of $\mathcal{X}$ at scale $r_\Delta$, there exists $a\in S$ such that $x^\star\in B_\infty^+(a,r_\Delta) \cap \mathcal{X}$. Hence $a\preceq x^\star$ and $\|x^\star-a\|_\infty\le r_\Delta$. By \autoref{assump : Lipschitz one sided},
\[
\Delta(a)\ =\ f^\star-f(a)\ \le\ L\|x^\star-a\|_\infty\ \le\ Lr_\Delta\ =\ \frac{\Delta_f}{4}.
\]
By definition of $\Delta_f$, every suboptimal point has gap at least $\Delta_f$. Therefore $a$ cannot be suboptimal, and thus $a\in\mathcal{X}^\star$. This proves that $S\cap\mathcal{X}^\star\neq\emptyset$.

The second statement is then immediate from standard finite-armed bandit theory: once the action set is reduced to the finite set $S$, any gap-dependent algorithm such as UCB satisfies a regret bound of order $\sum_{a\in S\setminus\mathcal{X}^\star}(\log T)/\Delta(a)$ in expectation; see, e.g., \citet[Chapter~7]{lattimore2020bandit}. Since $\Delta(a)\ge \Delta_f$ for every $a\in S\setminus\mathcal{X}^\star$, the displayed upper bound follows.
\end{proof}

\begin{remark}
Theorem~\ref{thm:onesided-positive-gap} isolates a regime in which one-sided Lipschitz bandits are genuinely easier than their symmetric counterparts. Under global Lipschitz continuity on a connected domain, a positive global gap is typically impossible unless the function is flat on an entire component. In contrast, \autoref{assump : Lipschitz one sided} allows such instances, and once a positive gap exists, the problem reduces to identifying an optimal representative in a finite forward cover. If $\Delta_f$ is unknown, one can combine the same idea with a geometric grid of candidate scales or a doubling scheme, at the price of additional logarithmic factors.
\end{remark}

\section{Results restated from the Literature} \label{app : restated from litt}

This section is dedicated to restating results from the literature. This allows us to restate the results as presented in previous work and walk the reader through changes in notations and settings required to be applied in this work, when needed, and clear from context, we also take the liberty of directly stating theorems with different notations.

\subsection{Lower bounds on the regret} \label{appendix:extensions}

We restate here results from \cite{shekhar2022instance}, which constitute the building blocks for the proof of our lower bounds. These results are stated for the Lipschitz bandit problem, where the goal is to design an adaptive querying strategy to optimize an unknown $L$-Lipschitz objective function $f$ via noisy zeroth-order queries.

Before restating the main result we are interested in, a lower bound on the regret in lipschitz bandits, we also restate for the sake of completeness the definition of an $a_0$ consistent algorithm $\mathcal{A}$.

\begin{definition}[$a_0$-consistency]
An algorithm $\mathcal{A}$ is said to be $a_0$-consistent over a function class $\mathcal{F}$, if for all $a > a_0$ and $f \in \mathcal{F}$, the following holds:
lim\begin{equation} \lim_{n \rightarrow \infty} \frac{\Expectation \left [\mathcal{R}_n (\mathcal{A},f) \right ]}{n^a}=0 \end{equation}
\end{definition}
Where $\Expectation\left [\mathcal{R}_n (\mathcal{A},f) \right ]$ denotes the pseudo-regret incured by algorithm $\mathcal{A}$ over $n$ timesteps when it is applied to instance $f$.

We restate here the results from \cite{shekhar2022instance} that we use in our proofs.

\begin{definition}
\label{def:lipschitz-complexity}
Let $f$ be a $(1-\lambda) L$-Lipschitz function for some $\lambda \in (0,1)$. Fix a $\Delta>0$, and introduce the set $\mathcal{Z}_k := \{ x\in \mathcal{X} |  2^k \Delta \leq f(x^*) - f(x) < 2^{k+1}\Delta\}$. Introduce the radius $w_k = 3 \times 2^k \Delta/(\lambda L)$, and let $m_k$ denote the $2w_k$ packing  number of the set $\mathcal{Z}_k$ for $k \geq 0$. Then, we can define the following complexity term: 
        \begin{align}
            \label{eq:lipschitz-complexity} 
            \ \mathcal{C}_{Lip} (\Delta,  L, \lambda) := \sum_{k \geq 0} \frac{m_k}{2^{k+2}\Delta} > \frac{m_0}{4\Delta}. 
        \end{align}
\end{definition}

\begin{proposition}
\label{prop:lipschitz-regret}
    For a $(1-\lambda)L$-Lipschitz function $f$, the expected regret of an $a_0$-consistent~(for the family of $L$-Lipschitz functions) algorithm $\mathcal{A}$ satisfies: 
    \begin{align}
        \mathbb{E} \left [ \mathcal{R}_n ( \mathcal{A}, f ) \right ] = \Omega \left ( \sigma^2 \mathcal{C}_{Lip} ( n^{-(1-a)}, L, \lambda ) \right ).  
    \end{align}
\end{proposition}

%\newpage 

%\input{checklist}
\end{document}